\documentclass[journal]{IEEEtai}

\usepackage[colorlinks,urlcolor=blue,linkcolor=blue,citecolor=blue]{hyperref}

\usepackage{color,array}
\usepackage[dvipsnames]{xcolor}

\usepackage{graphicx}
\usepackage{graphicx}%
\usepackage{multirow}%
\usepackage{amsmath,amssymb,amsfonts}%
\usepackage{amsthm}%
\usepackage{mathrsfs}%
\usepackage{xcolor}%
\usepackage{textcomp}%
\usepackage{manyfoot}%
\usepackage{booktabs}%
\usepackage{algorithm}%
\usepackage{algorithmic}
\usepackage{listings}%
\usepackage{subfigure}
\usepackage{float}
\usepackage{wrapfig}

\newcommand{\innerproductcomma}[2]{\langle #1, #2 \rangle}

\newcommand{\reone}[1]{{#1}}

\newcommand{\fix}[1]{{#1}}

\newtheorem{theorem}{Theorem}
\newtheorem{proposition}[theorem]{Proposition}%
\newtheorem{definition}{Definition}%

\begin{document}

\title{DeepHGCN: Toward Deeper Hyperbolic Graph Convolutional Networks} 

\author{Jiaxu~Liu,
        Xinping~Yi,
        and~Xiaowei~Huang
\thanks{The manuscript is submitted for review on October 10, 2023. Revised on May 22, 2024 and July 16, 2024. Accepted on August 03, 2024. (Corresponding author: Xinping Yi)}
\thanks{J. Liu and X. Huang are with the School of Electrical Engineering, Electronics and Computer Science, University of Liverpool, Liverpool, L69 3GJ UK; e-mail: \texttt{\{jiaxu.liu, xiaowei.huang\}@liverpool.ac.uk}.}
\thanks{X. Yi is with the National Mobile Communications Research Laboratory, Southeast University, Nanjing, China; email: \texttt{xyi@seu.edu.cn}.}
\thanks{J. Liu is supported by the Liverpool-CSC scholarship [202208890034]. X. Huang is supported by the UK EPSRC through End-to-End Conceptual Guarding of Neural Architectures [EP/T026995/1].}
}

\markboth{IEEE Transactions on Artificial Intelligence, Vol. 00, No. 0, Month 2020}
{Jiaxu Liu \MakeLowercase{\textit{et al.}}: Bare Demo of IEEEtai.cls for IEEE Journals of IEEE Transactions on Artificial Intelligence}

\maketitle

\begin{abstract}
Hyperbolic graph convolutional networks (HGCNs) have demonstrated significant potential in extracting information from hierarchical graphs. However, existing HGCNs are limited to shallow architectures due to the computational expense of hyperbolic operations and the issue of over-smoothing as depth increases. Although treatments have been applied to alleviate over-smoothing in GCNs, developing a hyperbolic solution presents distinct challenges since operations must be carefully designed to fit the hyperbolic nature. Addressing these challenges, we propose DeepHGCN, the first deep multi-layer HGCN architecture with dramatically improved computational efficiency and substantially reduced over-smoothing. DeepHGCN features two key innovations: (1) a novel hyperbolic feature transformation layer that enables fast and accurate linear mappings, and (2) techniques such as hyperbolic residual connections and regularization for both weights and features, facilitated by an efficient hyperbolic midpoint method. Extensive experiments demonstrate that DeepHGCN achieves significant improvements in link prediction and node classification tasks compared to both Euclidean and shallow hyperbolic GCN variants.
\end{abstract}

\begin{IEEEImpStatement}
Graph-structured data presents unique challenges in machine learning due to its complex nature. Traditional Graph Convolutional Networks (GCNs) primarily use Euclidean space for node embeddings, which struggle to capture the intricacies of hierarchical graphs. While HGCNs offer a solution, they often face limitations due to computational challenges and over-smoothing issues, particularly in deeper architectures. In this article, we introduce DeepHGCN, a novel approach for deep HGCNs. This architecture employs a new hyperbolic feature transformation layer and several auxiliary techniques to address these challenges. Testing on datasets with various geometries shows that DeepHGCN outperforms both Euclidean-based and other hyperbolic GCN methods in tasks such as link prediction and node classification. This demonstrates the efficacy of our approach in enhancing hierarchical graph-structured learning across various domains.
\end{IEEEImpStatement}



\section{Introduction}
Graph convolutional networks (GCN) \cite{defferrard2016convolutional,kipf2017semisupervised,hamilton2017inductive} have emerged as a promising approach for analyzing graph-structured data \cite{li2024guest}, \textit{e.g.,} social networks \cite{clauset2008hierarchical}, protein interaction networks \cite{zitnik2019evolution}, human skeletons \cite{yan2018spatial}, drug molecules \cite{duvenaud2015convolutional}, cross-modal retrieval \cite{li2024educross}, to name a few. Conventional GCN methods embed node representations into Euclidean latent space for downstream tasks. However, Bourgain's theorem \cite{linial1995geometry} indicates that the Euclidean space with arbitrary dimensions fails to embed hierarchical graphs with low distortion, suggesting the inadequacy of the Euclidean space to accommodate complex hierarchical data \cite{clauset2008hierarchical,krioukov2010hyperbolic,papadopoulos2012popularity,bai2024haqjsk}.

Recently, the hyperbolic space \textit{a.k.a.} Riemannian manifold of constant negative sectional curvature \cite{gromov1987hyperbolic, hamann2018tree, ungar2008gyrovector}, has gained increasing attention in processing non-Euclidean data. Since the exponentially expanding capacity of hyperbolic space satisfies the demand for hierarchical data that requires an exponential amount of branching space, embedding graphs to such a manifold naturally promotes learning hierarchical information. Based on two prevalent isomorphic models for hyperbolic space (Fig. \ref{Fig.illustration-ball}), \cite{ganea2018hyperbolic} and \cite{nickel2018learning} introduced the basic operations for constructing hyperbolic neural networks (HNN). Subsequently, the researchers generalized GCN operations to hyperbolic domains and derived a series of hyperbolic graph convolutional network (HGCN) variants \cite{liu2019hyperbolic, chami2019hyperbolic, gulcehre2018hyperbolic, zhang2021lorentzian, dai2021hyperbolic, chen2021fully}. These hyperbolic models are more capable of generating high-quality representations with low embedding dimensions, making them particularly advantageous in low-memory circumstances.

Despite their popularity, most HGCNs only achieve competitive performance with a 2-layer model. This limitation hinders their ability to effectively gather information from higher-order neighbors. However, developing a deeper HGCN model faces two main challenges. First, the computational complexity involved in hyperbolic operations, particularly feature transformation, prevents HGCNs from going deeper. Second, when more layers are added, node representations within the same connected component become increasingly indistinguishable.

This phenomenon is inherited from their GCN counterparts, known as \textit{over-smoothing} \cite{li2018deeper}, which could severely degrade the performance of multi-layer GCNs. For Euclidean GCNs, the over-smoothing issue has been defined and extensively studied in \cite{chen2020simple,xu2018representation,rong2019dropedge, zhou2021dirichlet,huang2022graph}, where the proposed techniques ensure that deep Euclidean GCNs outperform shallow counterparts as the depth increases. Given the hyperbolic representation, it is evidenced empirically that HGCNs still suffer from over-smoothing issues as network depth increases.

Aiming to address above challenges, in this paper, we propose a HGCN variant that is adaptive to depth variation, namely the DeepHGCN (Fig.~\ref{Fig.architecture.2}), by stacking a number of carefully constructed HGCN layers with computationally-efficient feature transformation, which can effectively prevent over-smoothing and deliver improved accuracy over state-of-the-art models with a deep architecture. Our contributions toward deep HGCNs are summarized as follows:

\fix{
\begin{itemize}
    \item \textbf{Scalable and Efficient Backbone.} Dealing with the computational complexity when increasing the depth of multi-layer HGCNs, we derive a novel hyperbolic fully connected layer that offers better efficiency and expressiveness for Poincaré feature transformation. In addition, M\"obius gyromidpoint \cite{shimizu2020hyperbolic, ungar2005analytic} is also carefully incorporated as an accurate hyperbolic midpoint method that serves as the basis for not only message aggregation, but also all hyperbolic operations in our proposed techniques within the DeepHGCN architecture.
    \item \textbf{Extensive Techniques.} To address the over-smoothing issue, we generalize the concept of Dirichlet energy to the hyperbolic space, effectively tracking the smoothness of hyperbolic embeddings. Guided by the measure of hyperbolic Dirichlet energy, DeepHGCN is specifically powered by three techniques, namely the initial residual, weight alignment, and feature regularization. Evidently, DeepHGCN effectively alleviates the over-smoothing problem occurred in the Poincar\'e ball and can be naturally generalized to the Lorentz model.
    \item \textbf{Experiments and Ablation Studies.} Results on benchmark datasets (Tab.~\ref{tb:summary-lpnc}-\ref{tb:summary-deep}) have validated the efficacy of our method for both node classification and link prediction under various layer settings. Additionally, ablation study in Sec.~\ref{sec:ablation_study} demonstrate that all techniques are necessary for mitigating the over-smoothing issue.
\end{itemize}

\noindent The code is available at \url{https://github.com/ljxw88/deephgcn}
}
\begin{figure}
    \centering
    \includegraphics[width=0.99\linewidth]{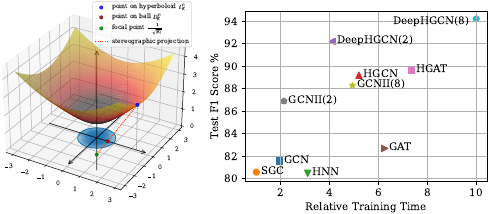}
    \vspace{-10pt}
    \caption{\textbf{Left}: Two prevalent hyperbolic models, isometric projection through the red line, where $\operatorname{P}_{\mathbb{D}\to\mathbb{L}}$: $\textcolor{red}{\bullet}\to \textcolor{blue}{\bullet}$ and $\operatorname{P}_{\mathbb{L}\to\mathbb{D}}$: $\textcolor{blue}{\bullet}\to \textcolor{red}{\bullet}$; \textbf{Right}: Performance over training time on \textit{Airport} in 5k epochs. DeepHGCN(2-layer) outperforms existing hyperbolic models and is more efficient. Increasing depth to DeepHGCN(8) bring further improvements.}
    \label{Fig.illustration-ball}
    \vspace{-10pt}
\end{figure}

\vspace{-10pt}
\section{Background}
\subsection{Brief Review of Riemannian Geometry}
\label{sec:bried-review-riemannian}
A \textit{manifold} $\mathcal{M}$ is a topological space that is locally Euclidean. The \textit{tangent space} $\mathcal{T}_\mathbf{x}\mathcal{M}$ is a real vector space of the same dimension as $\mathcal{M}$ attached to every point $\mathbf{x}\in \mathcal{M}$. All vectors in $\mathcal{T}_\mathbf{x}\mathcal{M}$ pass tangentially through $\mathbf{x}$. The \textit{metric tensor} $g_\mathbf{x}$ at point $\mathbf{x}$ defines an inner product on the associated tangent space, \textit{i.e.}, $g_\mathbf{x}:\mathcal{T}_\mathbf{x}\mathcal{M}\times \mathcal{T}_\mathbf{x}\mathcal{M} \to \mathbb{R}$. A \textit{Riemannian manifold} $(\mathcal{M}, g)$ is defined as a manifold equipped with Riemannian metric $g$. The metric tensor provides local geometric properties such as angles and curve lengths. A \textit{geodesic} is the shortest path between two points on the manifold. The \textit{exponential map} $\exp_\mathbf{x}: \mathcal{T}_\mathbf{x}\mathcal{M} \to \mathcal{M}$ defines a mapping of a tangent vector to a point on the manifold, and the \textit{logarithmic map} is the inverse $\log_\mathbf{x}: \mathcal{M} \to \mathcal{T}_\mathbf{x}\mathcal{M}$. Extended reviews are provided in Appendix~A.

\vspace{-10pt}
\subsection{Hyperbolic Geometry}
The hyperbolic space $\mathbb{H}^n_\kappa$ is a smooth Riemannian manifold with a constant sectional curvature $\kappa<0$ \cite{benedetti1992lectures}. It is usually defined via five isometric hyperbolic models, among which the $n$-dimensional Poincar\'e ball model $\mathbb{D}^n_\kappa= (\mathcal{D}^n_\kappa, g^{\mathbb{D}})$ and the Lorentz model (hyperboloid) $\mathbb{L}^n_\kappa = (\mathcal{L}^n_\kappa, g^{\mathbb{L}})$ are frequently used. The manifolds of $\mathbb{D}^n_\kappa$ are the projections of $\mathbb{L}^n_\kappa$ onto the $n$-dimensional space-like hyperplanes (Fig.~\ref{Fig.illustration-ball} left). In this paper, we align with \cite{ganea2018hyperbolic} and build our model upon the Poincar\'e ball model. As we still need the knowledge of the Lorentz model to complete our theories, we delegate the instruction and basic operations to Appendix~A and Tab.~\ref{tb:summary-operation}.

\textit{Poincar\'e Ball Model.} The $n$-dimensional Poincar\'e ball is defined as the Riemannian manifold $\mathbb{D}^n_\kappa = (\mathcal{D}^n_\kappa, g^{\mathbb{D}})$, with point set $\mathcal{D}^n_\kappa = \{\mathbf{x} \in \mathbb{R}^n : \|\mathbf{x}\| < -\frac{1}{\kappa}\}$ and Riemannian metric $g^{\mathbb{D}}_\kappa(\mathbf{x}) = (\lambda_\mathbf{x}^\kappa)^2 g^\mathbb{E}$, where $\lambda_\mathbf{x}^\kappa = \frac{2}{1+\kappa\|\mathbf{x}\|^2}$ (the conformal factor) and $g^\mathbb{E} = \mathbf{I}_n$. The Poincar\'e metric tensor induces various geometric properties, \textit{e.g.}, distances $d_\mathbb{D}^\kappa(\mathbf{x}, \mathbf{y})$, inner products $\innerproductcomma{\mathbf{u}}{\mathbf{v}}_\mathbf{x}^\kappa$, geodesics $\gamma_{\mathbf{x}, \mathbf{v}}(t)$ \cite{nickel2017poincare}, and more. The geodesics also induce the definition of exponential and logarithmic maps.
which are denoted
at $\mathbf{x}\in\mathbb{D}^n_\kappa$ as $\exp_\mathbf{x}^\kappa$ and $\log_\mathbf{x}^\kappa$, respectively. The M\"obius gyrovector space \cite{ungar2008gyrovector} offers an algebraic framework to treat the Poincar\'e coordinates as vector-like mathematical objects (gyrovectors). The gyrovectors are equipped with series of operations, \textit{e.g.} the vector addition $\oplus_\kappa$ and matrix-vector multiplication $\otimes_\kappa$. For brevity, we give instruction to Poincar\'e operations in Appendix~A and Tab.~\ref{tb:summary-operation}.

\section{Augmented HGCN Backbone}
\label{sec:hyperbolic-feature-trans}
\subsection{Efficient Feature Transformation}
Training multi-layer GCNs requires a fast and accurate linear transformation $\mathcal{F}:\mathbb{R}^{d_1} \to \mathbb{R}^{d_2}$ as backbone. However, an obvious deficiency of traditional hyperbolic linear layer is the propagation efficiency. In this paper, we propose a better unified linear layer $\mathcal{F}^\kappa_{\mathbb{D}}$ for efficient feature transformation within the Poincar\'e ball. With synthetic hyperbolic dataset, we show that $\mathcal{F}^\kappa_{\mathbb{D}}$ is faster and more expressive than both the naive HNN \cite{ganea2018hyperbolic} and PFC layer \cite{shimizu2020hyperbolic}. 

\begin{theorem}
\label{thm:poincare-fc-layer}
Given $\mathbf{h}^{} \in \mathbb{D}^{d_1}_\kappa$, Euclidean weight and bias parameter ${\mathbf{W}} \in \mathbb{R}^{d_2 \times d_1}$ and $\mathbf{b}_1, \mathbf{b}_2\in\mathbb{R}^{d_2}$. A more computational-efficient and expressive feature transformation $\mathcal{F}^\kappa_{\mathbb{D}}: \mathbb{D}^{d_1}_\kappa \to \mathbb{D}^{d_2}_\kappa$ within high dimensional Poincar\'e ball can be formulated by
\begin{align}
    &\mathcal{F}^\kappa_\mathbb{D}(\mathbf{h} ; \mathbf{W}, \mathbf{b}):= \frac{\phi({\mathbf{h}};\mathbf{W}, \mathbf{b})}{1 + \sqrt{|\kappa| \|\phi({\mathbf{h}};\mathbf{W}, \mathbf{b})\|^2 + 1 }}, \label{eq:poincare-fc-layer}
\end{align}
where $\phi(\cdot)$ is formulated as
\begin{align}
    \phi({\mathbf{h}};\mathbf{W}, \mathbf{b}) = \frac{2\sqrt{|\kappa|}\mathbf{W}\mathbf{h} + \mathbf{b}_1(1-\kappa\|\mathbf{h}\|^2)}{\sqrt{|\kappa|}(1 + \kappa\|\mathbf{h}\|^2)} + \mathbf{b}_2.\label{eq:poincare-fc-layer-2}
\end{align}
\end{theorem}
Given a point $\mathbf{x}$ on the hyperboloid, an arbitrary transformation matrix $\mathbf{W}\in\mathbb{R}^{d_2 \times (d_1 +1)}$ can be multiplied to $\mathbf{x}$. The Lorentzian \textit{re-normalization trick} ensures $\mathbf{x}$ still lies in the hyperboloid. Specifically, given $\mathbf{x}\in \mathbb{L}^{d_1}_\kappa (\subset\mathbb{R}^{d_1 + 1})$ and $\mathbf{W}\in \mathbb{R}^{d_2 \times (d_1 +1)}$, multiplying them gives $\mathbf{W}\mathbf{x}\in \mathbb{R}^{d_2}$ which is not essentially in the hyperboloid. To force the Lorentz constraint $\forall\mathbf{x}\in\mathbb{L}_\kappa:\innerproductcomma{\mathbf{x}}{\mathbf{x}}_\mathcal{L} = \frac{1}{\kappa}$, we re-normalize the time axis as $\sqrt{\|\mathbf{W}\mathbf{x}\|^2 - \frac{1}{\kappa}}$, such that the point set constraint is not violated. Thus a simple Lorentz transformation with re-normalization trick is expressed as
\begin{equation}
    \mathbf{x}' \leftarrow \begin{bmatrix}
        \sqrt{\|\mathbf{W}\mathbf{x}\|^2 - \frac{1}{\kappa}}\\ \mathbf{W}\mathbf{x}
    \end{bmatrix}. \label{eq:reformulated-matrix-layer}
\end{equation}
In a more general setting with bias parameter, $\mathbf{W}\mathbf{x}$ can be $\phi(\mathbf{x},\mathbf{W}, \mathbf{b}):\mathbb{L}^{d_1}\to\mathbb{R}^{d_2}$. It is easy to verify the Lorentzian constraint of Eq.~(\ref{eq:reformulated-matrix-layer}). 
Next, we give the bijection between an arbitrary point on the hyperboloid $\mathbf{z} = \begin{bmatrix}
    z_t \\
    \mathbf{z}_s
\end{bmatrix}\in \mathbb{L}^n_\kappa$ and the corresponding point on the Poincar\'e ball $\mathbf{x}\in \mathbb{D}^n_\kappa$ (Fig.~\ref{Fig.illustration-ball} left) as follows
\begin{align}
    & \mathbb{L}^n_\kappa\to \mathbb{D}^n_\kappa : \operatorname{P}_{\mathbb{L}\to \mathbb{D}}(\mathbf{z}) = \frac{\mathbf{z}_s}{1 + \sqrt{|\kappa|}z_t},\label{eq:L-to-D}\\
    & \mathbb{D}^n_\kappa\to \mathbb{L}^n_K : \operatorname{P}_{\mathbb{D}\to \mathbb{L}}(\mathbf{x}) = \begin{bmatrix}
        \frac{1 - \kappa \|\mathbf{x}\|^2 }{\sqrt{|\kappa|} + \kappa\sqrt{|\kappa|}\|\mathbf{x}\|^2}\\
        \frac{2\mathbf{x}}{1 + \kappa\|\mathbf{x}\|^2}
    \end{bmatrix}\label{eq:D-to-L}.
\end{align}
Therefore using Eq.~(\ref{eq:D-to-L}), given a point $\mathbf{x}\in \mathbb{D}^{d_1}_\kappa$, we derive the corresponding point on the hyperboloid $\hat{\mathbf{x}} = \begin{bmatrix}
    \frac{1 - \kappa \|\mathbf{x}\|^2 }{\sqrt{|\kappa|} + \kappa\sqrt{|\kappa|}\|\mathbf{x}\|^2}\\
    \frac{2\mathbf{x}}{1 + \kappa\|\mathbf{x}\|^2} 
\end{bmatrix}\in \mathbb{L}^{d_1}_\kappa(\subset \mathbb{R}^{d_1 + 1})$. Employing Eq.~(\ref{eq:reformulated-matrix-layer}), with a transformation matrix $\mathbf{W}\in\mathbb{R}^{d_2 \times (d_1 + 1)}$, we get $\mathbf{x}' \gets \begin{bmatrix}
    \sqrt{\|\phi(\hat{\mathbf{x}})\|^2 - \frac{1}{\kappa}} \\
    \phi(\hat{\mathbf{x}})
\end{bmatrix}\in \mathbb{L}^{d_2}_\kappa(\subset \mathbb{R}^{d_2 + 1})$, where $\phi(\hat{\mathbf{x}}) = \mathbf{W}\hat{\mathbf{x}} + \mathbf{b}$. Applying the reverse mapping Eq.~(\ref{eq:L-to-D}) gives
\begin{equation}
    \operatorname{P}_{\mathbb{L}\to \mathbb{D}}(\mathbf{x}') = \phi(\hat{\mathbf{x}}) \left(1 + \sqrt{|\kappa| \|\phi(\hat{\mathbf{x}})\|^2 - \operatorname{sgn}(\kappa)}\right)^{-1},\label{eq:poincare-fc-layer-explain}
\end{equation}
which gives us the form in Eq.~(\ref{eq:poincare-fc-layer}) as $\kappa<0$ in the Poincar\'e ball model. Since $\mathbf{W}$ is an arbitrary parameter in $\mathbb{R}$, we slice $\mathbf{W}\in \mathbb{R}^{d_2 \times (d_1 + 1)}$ as $\mathrm{concat}\left[W_t \in \mathbb{R}^{d_2}\| \mathbf{W}_s\in \mathbb{R}^{d_2 \times d_1}\right]$, therefore term $\mathbf{W}\hat{\mathbf{x}}$ can be reformulated as
\begin{align}
    \frac{2\sqrt{|\kappa|}\mathbf{W}_s\mathbf{x} + W_t(1- \kappa\|\mathbf{x}\|^2)}{\sqrt{|\kappa|} + \kappa\sqrt{|\kappa|} \|\mathbf{x}\|^2},\label{eq:poincare-fc-layer-2-explain}
\end{align}
which arrives at the form in Eq.~(\ref{eq:poincare-fc-layer-2}). Eq.~(\ref{eq:poincare-fc-layer-explain}-\ref{eq:poincare-fc-layer-2-explain}) together concludes Thm.~\ref{thm:poincare-fc-layer}.

\begin{proposition}
\label{prop:more-expressive-linear}
Given the $i$-th node representation $\mathbf{h}^{}_i\in\mathbb{D}^{d_1}_\kappa$, the hyperbolic feature transformation in Thm.~\ref{thm:poincare-fc-layer} $\mathbf{h}_i \leftarrow\mathcal{F}^\kappa_\mathbb{D}(\mathbf{h}^{}_i ; \mathbf{W}, \mathbf{b})$ yields more expressive node embeddings.
\end{proposition}
To verify Prop.~\ref{prop:more-expressive-linear}, we seek the connection of our approach to the formulation of PFC layer \cite{shimizu2020hyperbolic}, which is expressed as
\begin{align}
    &\mathbf{x}'\gets\mathcal{F}^\kappa_{\mathrm{PFC}}(\mathbf{x} ; \mathbf{W}, \mathbf{b}):= \frac{\boldsymbol{\omega}}{1+\sqrt{|\kappa|\|\boldsymbol{\omega}\|^2 + 1}} , \label{eq:pfc-formulation}\\ 
    &\text{where } \boldsymbol{\omega}:=\left(\frac{1}{\sqrt{|\kappa|}} \sinh \left(\sqrt{|\kappa|} \nu_j^\kappa(\mathbf{x}^{})\right)\right)_{j=1}^{d_2}.\label{eq:pfc-formulation-omega}
\end{align}
where $\nu_i^\kappa(\mathbf{x})$ is the unidirectional re-generalization of hyperbolic multinomial logistic regression. 
\begin{proposition}
    Given $\mathbf{x}\in \mathbb{D}_\kappa^{d_1}$ and $\mathbf{x}' = \mathcal{F}^\kappa_{\mathrm{PFC}}(\mathbf{x}) \in \mathbb{D}_\kappa^{d_2}$, the corresponding point of $\mathbf{x}'$ on the hyperboloid via stereographic projection is $\mathbf{h}' = \begin{bmatrix}
h_t \\ \mathbf{h}_s
\end{bmatrix}\in\mathbb{L}^{d_2 + 1}_\kappa$ where $\mathbf{h}_s = \boldsymbol{\omega}$ and $h_t = \sqrt{\|\boldsymbol{\omega}\|^2 - \frac{1}{\kappa}}$. 
\label{prop:lprojection-pfc}
\end{proposition}
\begin{proof}
We start with $\mathbf{h} \gets \mathcal{F}^\kappa_\mathrm{PFC}(\mathbf{x})$ (where $\mathbf{x}\in\mathbb{D}^{d_1}_\kappa$ and $\mathbf{h}\in\mathbb{D}^{d_2}_\kappa$ are respectively the input and output within the Poincar\'e ball) in Eq.~(\ref{eq:pfc-formulation}), applying stereographic projection in Eq.~(\ref{eq:D-to-L}) we obtain the corresponding point on the hyperboloid $\mathbf{h}^{\mathbb{L}} = \begin{bmatrix}
h_t \\ \mathbf{h}_s
\end{bmatrix}\in\mathbb{L}^{d_2 + 1}_\kappa$, where
\begin{align}
    \mathbf{h}_s &= 2\frac{\mathbf{h}}{1+\kappa\|\mathbf{h}\|^2}\\
    &= \frac{\frac{2\boldsymbol{\omega}}{\gamma}}{1+\kappa\|\frac{\boldsymbol{\omega}}{\gamma}\|^2} \quad \text{where $\gamma=1+\sqrt{1-\kappa\|\boldsymbol{\omega}\|^2}$}\\
    &= \frac{2\boldsymbol{\omega}\gamma}{\gamma^2 + \kappa\|\boldsymbol{\omega}\|^2}\\
    &= \frac{\boldsymbol{\omega}(2 + 2\sqrt{1-\kappa\|\boldsymbol{\omega}\|^2})}{1 + 2\sqrt{1-\kappa\|\boldsymbol{\omega}\|^2} + 1 - \kappa\|\boldsymbol{\omega}\|^2 + \kappa\|\boldsymbol{\omega}\|^2}\\
    &= \frac{\boldsymbol{\omega}(2 + 2\sqrt{1-\kappa\|\boldsymbol{\omega}\|^2})}{2 + 2\sqrt{1-\kappa\|\boldsymbol{\omega}\|^2}} = \boldsymbol{\omega},
\end{align}
and similarly
\begin{align}
    h_s = \frac{1 - \kappa \|\mathbf{h}\|^2 }{\sqrt{|\kappa|} + \kappa\sqrt{|\kappa|}\|\mathbf{h}\|^2} = \sqrt{\|\boldsymbol{\omega}\|^2 - \frac{1}{\kappa}}.
\end{align}
This concludes the proof.
\end{proof}

According to Prop.~\ref{prop:lprojection-pfc}, geometrically, the $\boldsymbol{\omega}$ in Eq.~(\ref{eq:pfc-formulation}) can be interpreted as a special feature transformation of the spatial component $\mathbf{h}_s$. The time component $h_t$ is a re-normalization according to $\mathbf{h}_s$ which stabilizes the point on the corresponding hyperboloid. The formulation in Eq.~(\ref{eq:pfc-formulation}) and, identically, Eq.~(\ref{eq:poincare-fc-layer}) ensures that \textbf{any} definition of $\boldsymbol{\omega}$ will \textbf{not} violate the Lorentzian constraint.

In essence, PFC breaks down to three core stages: \textit{1) Project Poincar\'e ball points to their hyperboloid equivalents via Eq.~(\ref{eq:D-to-L})}; \textit{2) Apply transformation $\boldsymbol{\omega}$ to the spatial component and \textbf{re-normalize} the time segment};  \textit{3) Reposition back to the Poincar\'e ball via Eq.~(\ref{eq:L-to-D})}. We hold the view that step-2, thanks to the re-normalization trick, with arbitrary $\boldsymbol{\omega}$, the transformed points will still adhere to the point set constraint (\textit{i.e.} $\{\mathbf{x} \in \mathbb{R}^n : \|\mathbf{x}\| < -\frac{1}{\kappa}\}$) after stereographic projection. Consequently, $\boldsymbol{\omega}$ can be an arbitrary linear transformation. Therefore, we argue to let $\boldsymbol{\omega} = \mathrm{MLP}(\mathbf{x})$ be a Euclidean neural net instead of Eq.~(\ref{eq:pfc-formulation-omega}), leading us to the expression in Eq.~(\ref{eq:poincare-fc-layer-2}).

\begin{figure}[t]
    \centering
    \includegraphics[width=0.99\linewidth]{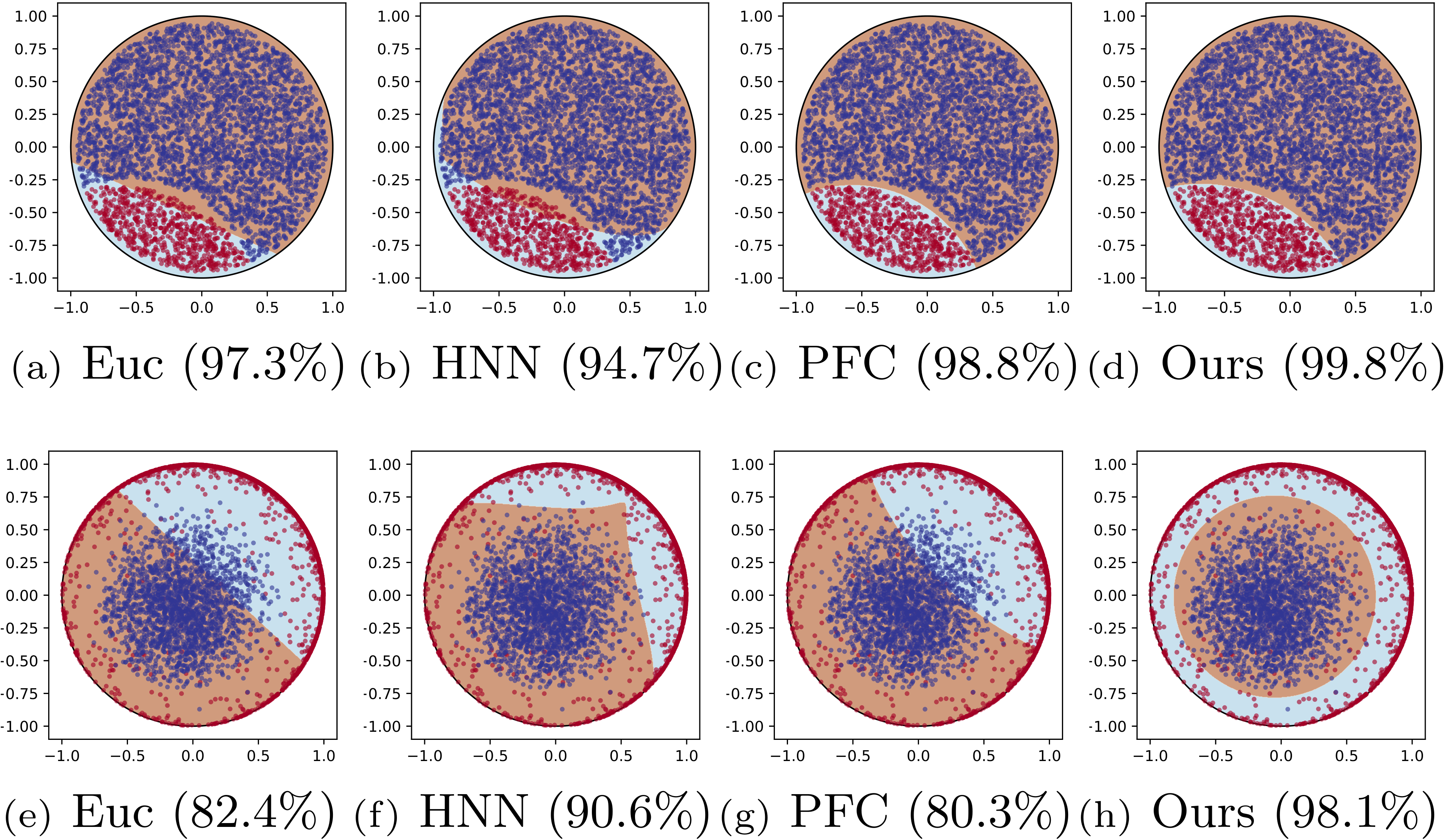}
  \label{Fig.cmp-binary-cls}
  \caption{Decision hyperplane of various feature transformations on synthetic binary classification tasks (Task \#1: (a)-(d)) and Task \#2: (e)-(h)).}
\end{figure}
\textbf{Performance Evaluation.} In Fig.~\ref{Fig.cmp-binary-cls}, we illustrate the decision hyperplane of different single transformation layers with randomly sampled two batches of points in $\mathbb{D}^2_\kappa$. We observe that the Eucliean hyperplane fail to adapt to non-linearity of the data samples due to the linear nature. The two hyperbolic baselines, limited by strict hyperbolic constraints on their formulations, also fails in fitting the data. Such a phenomenon is especially obvious on Fig.~\ref{Fig.cmp-binary-cls}.(b, f and g). In comparison, our reformulated FC layer offers the best performance on fitting the data samples.

\begin{table}[h]
\caption{Comparison of accuracy and calculation time (ms) of various hyperbolic FC layers. 2k+2k points are sampled within high dimensional Poincar\'e ball. We report mean $\pm$ std.}
\resizebox{\linewidth}{!}{\begin{tabular}{@{}ccc|cc|cc|cc@{}}
\toprule
\multirow{2}{*}{Task}     & \multicolumn{2}{c|}{Euclidean} & \multicolumn{2}{c|}{HNN} & \multicolumn{2}{c|}{PFC} & \multicolumn{2}{c}{\textbf{Ours}} \\ \cmidrule(l){2-9} 
                          & Acc       & Time       & Acc         & Time          & Acc         & Time       & Acc         & Time       \\ \midrule
\# 1 & $\text{97.3}_{\pm \text{3.2e-04}}$ & $\text{81.8}_{\pm \text{4.1e+01}}$ & $\text{95.6}_{\pm \text{5.8e-04}}$ & $\text{465.6}_{\pm \text{3.1e+02}}$ & $\text{99.8}_{\pm \text{1.7e-02}}$ & $\text{291.3}_{\pm \text{1.8e+01}}$ & $\text{100.0}_{\pm \text{2.6e-04}}$ & $\text{197.6}_{\pm \text{4.9e+00}}$ \\
\# 2 & $\text{82.6}_{\pm \text{4.3e-03}}$ & $\text{77.7}_{\pm \text{1.2e+00}}$ & $\text{90.6}_{\pm \text{2.2e-16}}$ & $\text{447.7}_{\pm \text{6.4e+01}}$ & $\text{94.8}_{\pm \text{5.7e-02}}$ & $\text{290.0}_{\pm \text{1.5e+01}}$ & $\text{98.2}_{\pm \text{5.3e-04}}$ & $\text{198.7}_{\pm \text{6.0e+00}}$ \\ \bottomrule
\end{tabular}

}
\label{tb:cmp-cls}
\end{table}

\reone{
\textbf{Computation Cost Analysis.} Assume the feature $\mathbf{h}$ is in $d_1$-dimension, and the feature transformation matrix $\mathbf{W}\in \mathbb{R}^{d_2 \times d_1}$. Theoretically, the complexity of HNN and our method are both $\mathcal{O}(d_1 \times d_2 + d_1 + 2d_2)$, which is similar to that of the Euclidean linear transformation $\mathbf{W}\mathbf{h}+\mathbf{b}$ ($\mathcal{O}(d_1 \times d_2 + d_2)$ complexity). Below, we provide detailed complexity analysis, and give the reason why the computation cost of HNN blows up while our method costs similar to standard MLP.

As detailed in Appendix~\ref{app:poincare-ball}, the HNN consists of exponential map, matrix-vector multiplication, and logarithmic map. We consider the simplest situation where the base is the north pole $\mathbf{0}$, consequently, the M\"obius additions $\oplus_\kappa$ are canceled and the conformal factor $\lambda_\mathbf{\mathbf{0}}^\kappa$ has $\mathcal{O}(1)$ complexity. Thus $\log_\mathbf{0}^\kappa (\mathbf{h})$ and $\exp_\mathbf{0}^\kappa (\mathbf{h})$ as in Eq.~(\ref{eq:logmap}-\ref{eq:expmap}) both have $\mathcal{O}(d)$ complexity, meaning the transformation $\exp_\mathbf{0}^\kappa\mathbf{W}\log_\mathbf{0}^\kappa (\mathbf{h}))$ is of $\mathcal{O}(d_1\times d_2 + d_1 + d_2)$ complexity. However, when considering bias translation, the $\oplus_\kappa$ must be accounted. As shown in Eq.~(\ref{eq:mobius_add}), $\oplus_\kappa$ basically consists inner products, norms and vector addition, all of them have $\mathcal{O}(d)$ complexity. The division has $\mathcal{O}(1)$ complexity, so overall $\oplus_\kappa$ is of $\mathcal{O}(4d)$ complexity. Therefore, the HNN $\exp_\mathbf{0}^\kappa(\mathbf{W}\log_\mathbf{0}^\kappa (\mathbf{h})) \oplus_\kappa \mathbf{b}$ has $\mathcal{O}(d_1\times d_2 + d_1 + 4d_2)$ complexity. Moreover, as evidenced in \cite{derczynski2020power,choudhary2022towards}, the hyperbolic trigonometric functions substantially increase the GPU burden on parallelism, which result in $4\sim 10$ times slower computation. With the repeat usage of $\mathrm{tanh}$ and $\mathrm{tanh}^{-1}$ in exp/log maps, the practical time consumption can be sufficiently higher than our theoretical analysis. Lastly, we analyze the complexity of our method. Our method belongs to the same family with PFC, hence similar complexities. In particular, the $\phi(\cdot)$ in Eq.~(\ref{eq:poincare-fc-layer-2}) consists of matrix/scalar-vector multiplications, norm and addition, the complexity is $\mathcal{O}(d_1 \times d_2 + d_1 + 2d_2)$. With the result $\phi(\mathbf{h};\mathbf{W},\mathbf{b})$, the complexity of Eq.~(\ref{eq:poincare-fc-layer}) (with only norm and devision) is $\mathcal{O}(d_1 \times d_2 + d_1 + 4d_2)$.

From above we conclude that, all methods can achieve an approximately $\mathcal{O}(d_1 \times d_2)$ complexity. The HNN, however, requires heavy usage of trigonometric funcions, which are computationally costly as evidenced in the literature, thus hindering its scalability. Our approach, instead, requires no hyperbolic trigonometric functions, thus leads to similar cost as Euclidean MLP. In Tab.~\ref{tb:cmp-cls}, we illustrate the performance of transformation layers and observe that our approach achieves the best accuracy on two simple classification tasks. Notably, our approach requires approximately $2$ times of the computation time of Euclidean linear layer, while other approaches require $3$ times and more. This coincides with our analysis.

}

\subsection{Efficient Message Aggregation}
\begin{definition}[M\"obius gyromidpoint \cite{ungar2008gyrovector}]
Given the gyrovectors $\{\mathbf{x}_i\in \mathbb{D}^d_\kappa\}_{i=1}^N$ and the weights $\{w_i\in \mathbb{R}\}_{i=1}^N$, the weighted gyromidpoint in the Poincar\'e ball $f_\mathrm{MG}^{\kappa}: \mathbb{D}_\kappa^{N\times d}\times \mathbb{R}^N \to \mathbb{D}^d_\kappa$ is defined as
\begin{align}
    f_\mathrm{MG}^{\kappa}(\{\mathbf{x}_i\}_{i=1}^N , \{w_i\}_{i=1}^N) 
    =\frac{1}{2} \otimes_\kappa \left( \frac{\sum_{i=1}^N w_i \lambda_{\mathbf{x}_i}^\kappa \mathbf{x}_i}{\sum_{i=1}^N |w_i| (\lambda_{\mathbf{x}_i}^\kappa - 1)} \right). \nonumber
\end{align}
\label{def:gyromidpoint}
\end{definition}
{

\begin{figure}[t]
    \centering
    \includegraphics[width=0.8\linewidth]{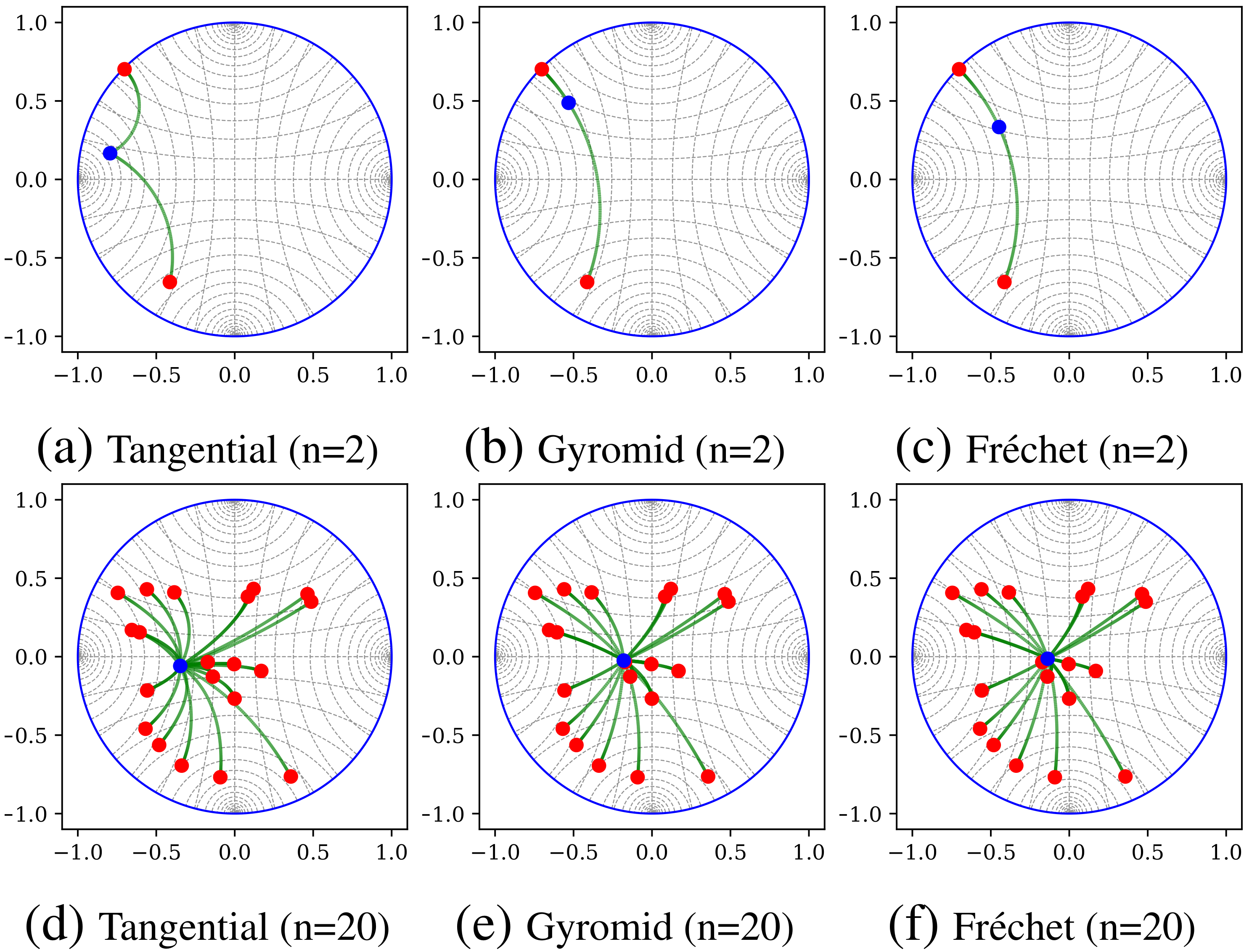}
  \caption{Compare tangential midpoint \cite{chami2019hyperbolic}, M\"obius gyromidpoint \cite{ungar2008gyrovector} and differentiable Fr\'echet mean \cite{Lou2020DifferentiatingTT} in the 2-dimension Poincar\'e disk. For each method we illustrate the weighted midpoint (blue) for double and multiple randomly sampled points (red) with randomly initialized weights.}
  \label{Fig.vismid}
\end{figure}
{
\begin{table}[t]
\caption{Compare various hyperbolic averaging methods on precision and calculation time (ms) vs the baseline (1000-iter Frech\'et mean). 4k points are randomly sampled within high dimensional Poincar\'e ball. We report mean $\pm$ std.}
\resizebox{\linewidth}{!}{\begin{tabular}{@{}cc|cc|cc|cc@{}}
\toprule
\multirow{2}{*}{Dim}    & Baseline                         & \multicolumn{2}{c|}{Tangential}                                     & \multicolumn{2}{c|}{\textbf{M\"obius Gyro}}                                           & \multicolumn{2}{c}{Frech\'et-Early}                                     \\ \cmidrule(l){2-8} 
                        & Time                             & MSE                                 & Time                          & MSE                                   & Time                          & MSE                                   & Time                          \\ \midrule
\multicolumn{1}{c|}{8}  & $\text{152.4}_{\pm \text{0.5}}$  & $\text{4.8e-5}_{\pm \text{4.7e-6}}$ & $\text{3.4}_{\pm \text{0.3}}$ & $\text{1.5e-6}_{\pm \text{1.5e-7}}$   & $\text{2.3}_{\pm \text{0.3}}$ & $\text{6.2e-29}_{\pm \text{2.8e-29}}$ & $\text{2.3}_{\pm \text{0.0}}$ \\
\multicolumn{1}{c|}{16} & $\text{219.4}_{\pm \text{1.1}}$  & $\text{6.5e-5}_{\pm \text{4.7e-6}}$ & $\text{2.8}_{\pm \text{0.3}}$ & $\text{9.1e-7}_{\pm \text{6.8e-8}}$   & $\text{1.7}_{\pm \text{0.2}}$ & $\text{6.9e-30}_{\pm \text{1.9e-30}}$ & $\text{2.8}_{\pm \text{0.1}}$ \\
\multicolumn{1}{c|}{64} & $\text{371.1}_{\pm \text{15.1}}$ & $\text{3.5e-5}_{\pm \text{1.2e-6}}$ & $\text{3.3}_{\pm \text{0.4}}$ & $\text{2.4e-10}_{\pm \text{8.9e-12}}$ & $\text{2.2}_{\pm \text{0.3}}$ & $\text{3.1e-31}_{\pm \text{1.4e-31}}$ & $\text{3.9}_{\pm \text{0.2}}$ \\ \bottomrule
\end{tabular}}
\label{tab:cmp-midpoint-runtime}
\end{table}

\begin{table}[t]
\caption{Peak memory usage comparison (4K point, 64 dim, mean$\pm$std)}
\centering
\resizebox{0.99\linewidth}{!}{
\begin{tabular}{@{}l|llll@{}}
\toprule
Method         & Frct-MaxIter                           & Tangential                           & \textbf{M\"obius Gyro}                     & Frct-Early                          \\ \midrule
Peak Mem (KiB) & $\text{2727.8}_{\pm \text{82.4}}$ & $\text{536.9}_{\pm \text{24.6}}$ & $\textbf{411.1}_{\pm \textbf{85.6}}$ & $\text{2670.4}_{\pm \text{110.8}}$ \\ \bottomrule
\end{tabular}
}
\label{tab.memory_usage}
\end{table}
}

The mean operator is an essential building block for neural networks. In non-Euclidean geometries, the mean computation cannot be performed simply by averaging the inputs, as the averaged vector may be out of manifold. Basically, there are three types of generalized weighted mean to hyperbolic space that can be used to guarantee the summed vectors on the manifold and to be differentiable, namely, the tangential aggregation \cite{chami2019hyperbolic}, hyperbolic gyromidpoint \cite{ungar2008gyrovector} and differentiable Frech\'et mean \cite{Lou2020DifferentiatingTT}. In this work, we employ the hyperbolic gyromidpoint as the unified faster and accurate mean operator, defined in Def.~\ref{def:gyromidpoint}. With Def.~\ref{def:gyromidpoint}, we define the convolution for hyperbolic node feature $\mathbf{h}^{(l)}\in \mathbb{D}^{|\mathcal{V}|\times d}_\kappa$ upon M\"obius gyromidpoint. Given the augmented normalized adjacency matrix $\tilde{\mathbf{P}}\in \mathbb{R}^{|\mathcal{V}|\times |\mathcal{V}|}$, we define our message aggregation as
\begin{equation}
    f_\mathrm{NA}^{\kappa} (\tilde{\mathbf{P}},\mathbf{h}) = \left(\frac{1}{2} \otimes_\kappa \left( \frac{\sum_{j=1}^d \tilde{\mathbf{P}}_i \lambda_{\mathbf{h}_j}^\kappa \mathbf{h}_j}{\sum_{j=1}^d |\tilde{\mathbf{P}}_{i}| (\lambda_{\mathbf{h}_j}^\kappa - 1)} \right)\right)_{i=1}^{|\mathcal{V}|}.
    \label{eq:hyperbolic-neighbour-agg}
\end{equation}
}



\textbf{Performance Evaluation.} We demonstrate the precision and calculation time of three methods in Tab.~\ref{tab:cmp-midpoint-runtime}, regarding the standard Fr\'echet mean \cite{karcher1987riemannian, karcher2014riemannian} as baseline. We observe: (1) the tangential aggregation is efficient but inaccurate (comparing Fig.~\ref{Fig.vismid}.a to \ref{Fig.vismid}.b and \ref{Fig.vismid}.c), where the resulting midpoint is no longer on the geodesic between the inputs; (2) computing the differentiable Fr\'echet mean \cite{Lou2020DifferentiatingTT} requires an iterative solver, and obtaining an accurate result requires considerable computation; (3) M\"obius gyromidpoint gives a close solution to the Fr\'echet mean while significantly reducing the complexity. Furthermore, we provide \textit{memory evaluation} in Tab.~\ref{tab.memory_usage}, where we report peak memory usage comparison on the same synthetic dataset under $1000$ runs. The employed gyromidpoint approach exhibits the lowest memory cost, which highlights its advantage in computation efficiency and scalability, allowing us to further expand our model to a multi-layer scheme.


\section{Toward Deeper Hyperbolic GCN}

\subsection{Over-Smoothness Analysis}
First, we derive the hyperbolic Dirichlet energy $f^{\mathbb{D}}_\mathrm{DE}(\cdot)$ as a measure of smoothness for the Poincar\'e embeddings. 
\begin{definition}
Given the embedding $\mathbf{h} = \{\mathbf{h}^{}_i  \in \mathbb{D}^{d}_\kappa\}_{i=1}^{|\mathcal{V}|}$, the hyperbolic Dirichlet energy $f^{\kappa}_\mathrm{DE}(\mathbf{h})$ is defined as
\begin{align}
\frac{1}{2} \sum_{(i,j)\in \mathcal{E}} {d_{\mathbb{D}}^\kappa} \left( 
\exp^\kappa_\mathbf{o}\left(\frac{\log_\mathbf{o}^\kappa(\mathbf{h}_i)}{\sqrt{1+d_i}}\right)  ,\exp^\kappa_\mathbf{o}\left(\frac{\log_\mathbf{o}^\kappa(\mathbf{h}_j)}{\sqrt{1+d_j}}\right)
\right)^2 \nonumber  \label{eq:dirichlet-energy-definition-hyperbolic},
\end{align}
where $d_{i/j}$ denotes the node degree of node $i/j$. The distance $d_\mathbb{D}^\kappa(\mathbf{x}, \mathbf{y})$ between two points $\mathbf{x}, \mathbf{y}\in \mathbb{D}$ is the geodesic length, we detail the closed form expression in Appendix~A.
\label{def:dirichlet-energy-definition}
\end{definition}

\begin{figure}[t]
\centering
\subfigure[Existing HGCN schematic.]{
\label{Fig.architecture.1}
\includegraphics[width=0.85\linewidth]{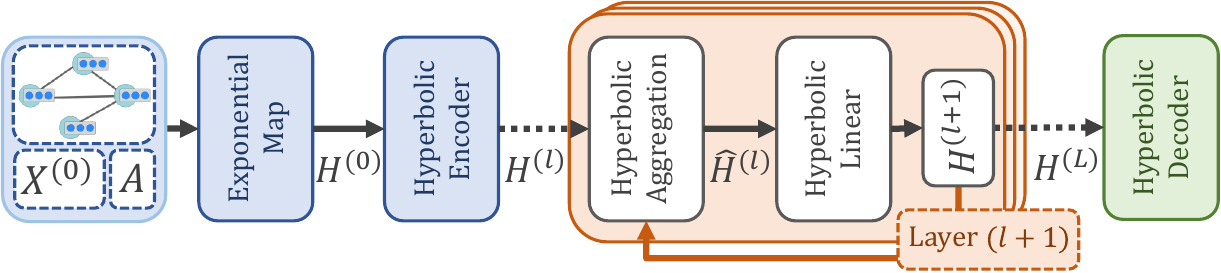}}
\subfigure[The DeepHGCN schematic.]{
\label{Fig.architecture.2}
\includegraphics[width=0.99\linewidth]{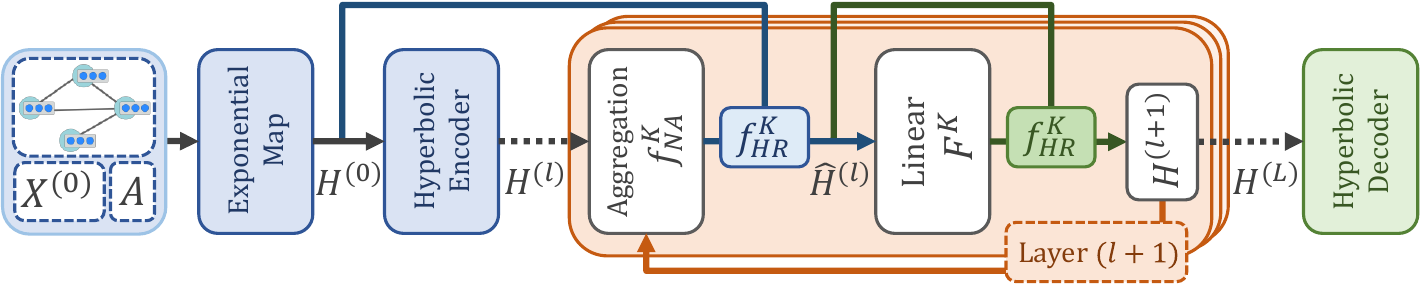}}
\caption{Comparison between the existing HGCN architecture and the proposed DeepHGCN. At the $l$-th layer, (a) performs linear transformation directly after the aggregation and regards the transformed feature as next layer's input, causing over-smoothing as $l$ increases; (b) performs hyperbolic residual connection after aggregation and linear layer to alleviate over-smoothing, such that the hyperbolic residual operator retains the feature on the manifold and the global hyperbolic geometry is preserved.}
\label{Fig.architecture-overview}
\end{figure}

\begin{proposition}
Hyperbolic message aggregation reduces the Dirichlet energy. \textit{i.e.}
    $f_\mathrm{DE}^\kappa(\tilde{\mathbf{P}}\otimes_\kappa\mathbf{h}^{(l)})\le f_\mathrm{DE}^\kappa (\mathbf{h}^{(l)})$.
\label{lemma:shrinking-property}
\end{proposition}

\noindent The Dirichlet energy of node representation in each layer can be viewed as the weighted sum of the distances between normalized node pairs. Prop.~\ref{lemma:shrinking-property} indicates that the energy of node representation will decay after the aggregation step. When multiple aggregations in HGCN are performed, the energy will converge to zero, indicating lower embedding expressiveness, which may lead to oversmoothing.

\subsection{Hyperbolic Graph Residual Connection}
We define the hyperbolic graph residual connection $f_\mathrm{HR}^\kappa$ upon the notion of M\"obius gyromidpoint. Given the node embedding matrices $\mathbf{h}^1, \mathbf{h}^2 \in \mathbb{D}^{|\mathcal{V}|\times d}_K$ and residual weight coefficients $w^1, w^2\in \mathbb{R}$, define
\begin{equation}
\begin{aligned}
    &f_\mathrm{HR}^{\kappa} (\mathbf{h}^1, \mathbf{h}^2; w^1, w^2)\\
    &= \left(\frac{1}{2}\otimes_\kappa \left( \frac{w^1\lambda_{\mathbf{h}^1_i}^\kappa\mathbf{h}^1_i + w^2\lambda_{\mathbf{h}^2_i}^\kappa\mathbf{h}^2_i}{|w^1|(\lambda_{\mathbf{h}^1_i}^\kappa-1) + |w^2|(\lambda_{\mathbf{h}^2_i}^\kappa-1)} \right)\right)_{i=1}^{|\mathcal{V}|}.
\end{aligned}
\label{eq:hyperbolic-residual}
\end{equation}
This operation ensures the feature after residual connection still lies in the Poincar\'e ball. One also recovers the arithmetic mean as $\kappa\to 0$. 

\subsubsection{Hyperbolic Initial Residual}
The graph residual connection is firstly introduced in the standard GCN \cite{kipf2017semisupervised}, in which the current layer representation $\sigma(\tilde{\mathbf{P}}\mathbf{h}^{(l)}\mathbf{W}^{(l)})$ is connected to previous layer $\mathbf{h}^{(l)}$ to facilitate a deeper model. Some works \cite{klicpera2018predict, chen2020simple} empirically find the efficacy of adding residual connections to initial layer $\mathbf{h}^{(0)}$. It was also claimed in \cite{zhou2021dirichlet} that residual connections to both initial layer $\mathbf{h}^{(0)}$ and previous layer $\mathbf{h}^{(l)}$ can prevent the Dirichlet energy being below the lower energy limit that causes over-smoothing. Based on these studies, we formulate the \textit{hyperbolic residual layer} by
\begin{equation}
    \hat{\mathbf{h}}^{(l)} = f_\mathrm{HR}^{\kappa}\left( \mathbf{h}^{(t)} ,  f_\mathrm{NA}^{\kappa}(\tilde{\mathbf{P}}, \mathbf{h}^{(l)}); \alpha_l, 1 - \alpha_l  \right),
\label{eq:hyperbolic-initial-residual}
\end{equation}
where $t$ in this case can be $0$ (initial layer) or $l-1$ (previous layer). Empirically, we find that adding previous residual does not contribute to the overall performance than adding only initial residual. Therefore, we set $t=0$ in all case of our study. $f_\mathrm{HR}$ and $f_\mathrm{NA}$ are respectively defined in Eq. (\ref{eq:hyperbolic-residual}) and (\ref{eq:hyperbolic-neighbour-agg}). The hyperparameter $\alpha_l$ for the $l$-th layer indicates the proportion of residual representation that the current layer retains. In practice, this value could be relatively small so it does not conceal the variation of embedding whilst alleviating over-smoothing.

\subsubsection{Hyperbolic Weight Alignment}

It is shown in \cite{kipf2017semisupervised, chen2020simple}, however, that simply applying residual operation to initial/previous layer only partially relieves the over-smoothing issue and still degrades the performance when more layers are stacked. To fix such deficiency, \cite{chen2020simple} borrows the idea from ResNet \cite{he2016deep} to align the weight matrix $\mathbf{W}^{(l)}$ in each layer to an identity matrix $\mathbf{I}$, which can be formulated as
\begin{align}
    {\mathbf{h}}^{(l+1)} & = \hat{\mathbf{h}}^{(l)} \left(\beta_l \mathbf{W}^{(l)} + (1-\beta_l)\mathbf{I}_{|\mathcal{V}|}\right)\label{eq:identity-mapping}\\
    & = \beta_l (\hat{\mathbf{h}}^{(l)} \mathbf{W}^{(l)}) + (1-\beta_l)\hat{\mathbf{h}}^{(l)} \label{eq:identity-map-expand},
\end{align}
where $\hat{\mathbf{h}}^{(l)}\in \mathbb{D}^{|\mathcal{V}|\times d}_K$ is the aggregated feature. If $\beta_l$ is sufficiently small, the model ignores the weight matrix and simulates the behaviour of APPNP \cite{gasteiger2018predict}. Further, it forces a small $\|\mathbf{W}\|$, which implies a small $s^L$ ($s$ is the maximum singular value of $\mathbf{W}^{(L)}$). According to \cite{oono2019graph}, the loss of information in the $L$-layer GCN is relieved. 

To generalize Eq.~(\ref{eq:identity-mapping}-\ref{eq:identity-map-expand}) to hyperbolic setting, one can leverage the power of M\"obius gyromidpoint as Def.~\ref{def:gyromidpoint}. Following the expanded Eq.~(\ref{eq:identity-map-expand}), whose form is the weighted average between $\hat{\mathbf{h}}^{(l)}$ and $\hat{\mathbf{h}}^{(l)}\mathbf{W}^{(l)}$, thus midpoint can be employed to calculate the weighted mean between hyperbolic representation $\hat{\mathbf{h}}^{(l)}$ and the hyperbolic transformed representation $\mathbf{W}\otimes_K\hat{\mathbf{h}}^{(l)}$. Define the reformulation
\begin{equation}
    \mathbf{h}^{\star (l+1)} = f_\mathrm{HR}^{\kappa}\left(\mathbf{W}^{(l)}\otimes_K\hat{\mathbf{h}}^{(l)}, \hat{\mathbf{h}}^{(l)}; \beta_l, 1 - \beta_l\right),
\label{eq:identity-map-gyromid}
\end{equation}
where $f_\mathrm{HR}^{\kappa}$ is defined in Eq. (\ref{eq:hyperbolic-residual}). 
Finally, we augment vanilla M\"obius transformation $\mathbf{W}^{}\otimes_\kappa\hat{\mathbf{h}}^{}$ by our proposed $\mathcal{F}^K$ for the Poincar\'e ball-based model, which yields the formulation of proposed \textit{hyperbolic weight alignment}
\begin{equation}
    \mathbf{h}^{(l+1)} = f_\mathrm{HR}^{\kappa}\left(\mathcal{F}^\kappa(\hat{\mathbf{h}}^{(l)} ; \mathbf{W}^{(l)}), \hat{\mathbf{h}}^{(l)}; \beta_l, 1 - \beta_l\right).
    \label{eq:hyperbolic-identity-map}
\end{equation}

\subsubsection{Hyperbolic Feature Regularization}
Notably, \cite{khrulkov2020hyperbolic} empirically finds that the learnt embeddings near the boundary of the hyperbolic ball are easier to classify. To push the nodes to the boundary, \cite{yang2022hrcf} proposed a regularization that first identifies a root node in hyperbolic space and then encourages the embeddings to move away from the root. We employ a similar approach to regularize the training of DeepHGCN. Firstly, the root node can be identified via gyromidpoint in Def.~\ref{def:gyromidpoint} as
\begin{equation}
    \mathbf{h}_\mathrm{root}^{(l)} = \frac{1}{2} \otimes_\kappa \left( \frac{\sum_{i=1}^{|V|} \lambda_{\mathbf{h}_i^{(l)}}^\kappa \mathbf{h}_i^{(l)}}{\sum_{i=1}^{|V|} (\lambda_{\mathbf{h}_i^{(l)}}^\kappa - 1)} \right),
    \label{eq:feature-reg-begin}
\end{equation}
where $\mathbf{h}_i$ is the row-vector of embedding matrix $\mathbf{H}^{(l)}$. Each node is aligned with the root by $\mathbf{\bar{h}}^{(l)}_i = \mathbf{h}^{(l)}_i \ominus_\kappa \mathbf{h}_\mathrm{root}^{(l)}$. By the definition of the Poincar\'e ball we have $\|\mathbf{\bar{h}}^{(l)}_i\| < \frac{1}{|\kappa|}$. Pushing the aligned embeddings to the boundary deduces the norm of aligned embeddings $\|\mathbf{\bar{h}}^{(l)}_i\| \to \frac{1}{|\kappa|}$, which is equivalent to minimizing the inverse of the norm. The regularization term can therefore be defined as the inverse quadratic mean of the norms, formulated as
\begin{equation}
    \mathcal{L}_\mathrm{reg}^{(l)} = \frac{1}{\sqrt{\frac{1}{|V|} \sum_{i=1}^{|V|} \| \mathbf{\bar{h}}^{(l)}_i \|^2 }}.
    \label{eq:feature-reg-end}
\end{equation}
In practice, computing $\mathcal{L}_\mathrm{reg}^{(l)}$ for a $L$-layer deep network is costly. Since the embedding is densely connected through each layer in DeepHGCN, we can consider only the last layer regularization $\mathcal{L}_\mathrm{reg}^{(L)}$. Hence, the optimization target is formulated as $\mathcal{L} + \gamma \mathcal{L}_\mathrm{reg}^{(L)}$ where $\gamma$ is a hyperparameter.

\subsubsection{Rethinking DeepHGCN through Dirichlet Energy}
\label{sec:rethinking}
As per Def.~\ref{def:dirichlet-energy-definition}, Dirichlet energy is essentially the weighted average of distances between normalized node pairs. To prevent over-smoothing, we want the energy of the last layer to be sufficiently large. On one hand, the initial residual and weight alignment ensure that the final representation contains at least a portion of the initial and previous layers. Since the energy of the starting layers is high, the residual connection mitigates the degradation of energy and retains the energy of the final layer at the same level as the previous layers. On the other hand, hyperbolic feature regularization encourages the node representation to move away from the center. This increases the distance between nodes and thus also alleviates the energy degradation.

\begin{algorithm}[tb]
\caption{DeepHGCN forward propagation}
\label{alg:algorithm}
\textbf{Input}: graph $\mathcal{G} = (\mathcal{V},\mathcal{E})$; node embeddings $\{\mathbf{x}_i\}_{i=1}^{|\mathcal{V}|}$; number of layers $L$; feature dim $d_f$, hidden dim $d_h$ and class dim $d_c$; activation $\sigma(\cdot)$; hyper-params $\{\alpha_l, \beta_l\}_{l=1}^L$ and $\gamma$\\
\textbf{Parameter}: $\{\mathbf{W}^{(l)}\}_{l=0}^L$ and $\{\mathbf{b}^{(l)}\}_{l=1}^L$\\
\textbf{Output}: Loss for back-propagation

\begin{algorithmic}[1] 
\STATE generate hyperbolic representation $\{\mathbf{h}_i^{(0)}\}_{i=1}^{|\mathcal{V}|}$ via exponential map Eq.~(\ref{eq:expmap})
transform $\mathbf{h}^{(0)} \to \mathbf{h}^{(1)}$ from $d_f$ to $d_h$ via Eq.~(\ref{eq:poincare-fc-layer})
\FOR{$l=0$ to $L$}
\STATE neighborhood aggregation on $\mathbf{h}^{(l)}$ via Eq.~(\ref{eq:hyperbolic-neighbour-agg})
\STATE residual connect $\mathbf{h}^{(1)}$ and $\mathbf{h}^{(l)}$ via Eq~(\ref{eq:hyperbolic-initial-residual})
\IF {$l==L$}
\STATE break the loop
\ENDIF
\STATE transform $\mathbf{h}^{(l)} \to \mathbf{h}^{(l+1)}$ within $\mathbb{D}^{d_h}$ via Eq.~(\ref{eq:poincare-fc-layer})
\STATE weight alignment on $\mathbf{h}^{(l)}$ and $\mathbf{h}^{(l+1)}$ via Eq.~(\ref{eq:hyperbolic-identity-map})
\ENDFOR
\STATE compute task-oriented loss $\mathcal{L}$ and feature regularization $\mathcal{L}_\text{reg}$ via Eq.~(\ref{eq:feature-reg-end}) with final embedding $\mathbf{H}^{(L)}$\;
\RETURN{} $\mathcal{L} + \gamma\mathcal{L}_\mathrm{reg}$
\end{algorithmic}
\label{alg:deephgcn-forward}
\end{algorithm}

\subsection{Training}
The DeepHGCN forward pass is described in Alg.~\ref{alg:deephgcn-forward}.
For link prediction (LP) and node classification (NC) task, we employ the same downstream decoders and loss in \cite{chami2019hyperbolic}, \textit{i.e.} Fermi-Dirac decoder for LP and cross entropy for NC. 
Since all DeepHGCN parameters are resided in Euclidean space, a standard optimizer (\textit{e.g.} Adam \cite{kingma2014adam}) instead of Riemannian optimizers \cite{becigneul2018riemannian} can be leveraged.

\section{Empirical Results}


\begin{table*}[htb]
\centering
\caption{ROC AUC results (\%) for Link Prediction (LP) and Accuracy (\%) for Node Classification (NC). $\delta$ refers to Gromovs $\delta$-hyperbolicity. A graph is more hyperbolic as $\delta\to 0$ and is a tree when $\delta = 0$. Results partially from \cite{zhu2020graph}.}
\resizebox{\textwidth}{!}{\begin{tabular}{@{}llcclcclcclcc@{}}
\toprule
\textbf{Dataset ($\delta$)} &  & \multicolumn{2}{c}{\textbf{Airport ($\delta = 1$)}}     &  & \multicolumn{2}{c}{\textbf{PubMed ($\delta =   3.5$)}}  &  & \multicolumn{2}{c}{\textbf{CiteSeer ($\delta =   5$)}}  &  & \multicolumn{2}{c}{\textbf{Cora ($\delta = 11$)}}       \\ \cmidrule(lr){3-4} \cmidrule(lr){6-7} \cmidrule(lr){9-10} \cmidrule(l){12-13} 
\textbf{Task}               &  & LP                         & NC                         &  & LP                         & NC                         &  & LP                         & NC                         &  & LP                         & NC                         \\ \midrule
MLP                         &  & 89.81${\pm 0.56}$          & 68.90${\pm 0.46}$          &  & 83.33${\pm 0.56}$          & 72.40${\pm 0.21}$          &  & 93.73${\pm 0.63}$          & 59.53${\pm 0.90}$          &  & 83.33${\pm 0.56}$          & 51.59${\pm 1.28}$          \\
HNN                         &  & 90.81${\pm 0.23}$          & 80.59${\pm 0.46}$          &  & 94.69${\pm 0.06}$          & 69.88${\pm 0.43}$          &  & 93.30${\pm 0.52}$          & 59.50${\pm 1.28}$          &  & 90.92${\pm 0.40}$          & 54.76${\pm 0.61}$          \\ \midrule
GCN                         &  & 89.31${\pm 0.43}$          & 81.59${\pm 0.61}$          &  & 89.56${\pm 3.66}$          & 78.10${\pm 0.43}$          &  & $82.56{\pm 1.92}$          & 70.35${\pm 0.41}$          &  & 90.47${\pm 0.24}$          & 81.50${\pm 0.53}$          \\
GAT                         &  & 90.85${\pm 0.23}$          & 82.75${\pm 0.36}$          &  & 91.46${\pm 1.82}$          & 78.21${\pm 0.44}$          &  & 86.48${\pm 1.50}$          & 71.58${\pm 0.80}$          &  & 93.17${\pm 0.20}$          & 83.03${\pm 0.50}$          \\
GraphSAGE                   &  & 90.41${\pm 0.53}$          & 82.19${\pm 0.45}$          &  & 86.21${\pm 0.82}$          & 77.45${\pm 2.38}$          &  & 92.05${\pm 0.39}$          & 67.51${\pm 0.76}$          &  & 85.51${\pm 0.50}$          & 77.90${\pm 2.50}$          \\
SGC                         &  & 89.83${\pm 0.32}$          & 80.59${\pm 0.16}$          &  & 94.10${\pm 0.12}$          & 78.84${\pm 0.18}$          &  & 91.35${\pm 1.68}$          & 71.44${\pm 0.75}$          &  & 91.50${\pm 0.21}$          & 81.32${\pm 0.50}$          \\ \midrule
HGNN                        &  & 96.42${\pm 0.44}$          & 84.71${\pm 0.98}$          &  & 92.75${\pm 0.26}$          & 77.13${\pm 0.82}$          &  & 93.58${\pm 0.33}$          & 69.99${\pm 1.00}$          &  & 91.67${\pm 0.41}$          & 78.26${\pm 1.19}$          \\
HGCN                        &  & 96.43${\pm 0.12}$          & 89.26${\pm 1.27}$          &  & 95.13${\pm 0.14}$          & 76.53${\pm 0.64}$          &  & 96.63${\pm 0.09}$          & 68.04${\pm 0.59}$          &  & 93.81${\pm 0.14}$          & 78.03${\pm 0.98}$          \\
HGAT                        &  & 97.86${\pm 0.09}$          & 89.62${\pm 1.03}$          &  & 94.18${\pm 0.18}$          & 77.42${\pm 0.66}$          &  & 95.84${\pm 0.37}$          & 68.64${\pm 0.30}$          &  & \textbf{94.02${\pm 0.18}$} & 78.32${\pm 1.39}$          \\
HyboNet                     &  & 97.30${\pm 0.30}$          & 90.90${\pm 1.40}$          &  & 95.80${\pm 0.20}$          & 78.00${\pm 1.00}$          &  & 96.70${\pm 0.80}$          & 69.80${\pm 0.60}$          &  & 93.60${\pm 0.30}$          & 80.20${\pm 1.30}$          \\ \midrule
\textcolor{blue}{DeepHGCN}  &  & \textbf{98.13${\pm 0.33}$} & \textbf{94.70${\pm 0.90}$} &  & \textbf{96.15${\pm 0.17}$} & \textbf{79.43${\pm 0.92}$} &  & \textbf{97.45${\pm 0.44}$} & \textbf{73.31${\pm 0.70}$} &  & 93.90${\pm 0.78}$          & \textbf{83.64${\pm 0.40}$} \\ \bottomrule
\end{tabular}}
\label{tb:summary-lpnc}
\end{table*}

\subsection{Experiment Setup}

\subsubsection{Datasets} Four homophilic benchmark datasets are considered for node classification and link prediction: Airport \cite{chami2019hyperbolic}, PubMed \cite{namata2012query}, CiteSeer \cite{giles1998citeseer} and Cora \cite{sen2008collective}, with statistics in Tab.~\ref{tb:summary-dataset}. Airport dataset is a transductive dataset in which nodes represent airports and edges indicate airline routes as derived from OpenFlights. The labels of nodes are determined by the population of the country to which the airport belongs. Cora, PubMed, and CiteSeer are benchmarks for citation networks, where nodes represent papers connected by citations. We report the hyperbolicity of each datatset (lower is more hyperbolic) as defined in \cite{gromov1987hyperbolic}. 
\begin{table}[htb]
\centering
\caption{Statistics of the general HGNN benchmark datasets.}
\resizebox{0.9\linewidth}{!}{\begin{tabular}{lrrrrr} 
\toprule
Dataset  & \# Nodes & \# Edges & Classes & Features & $\delta$  \\ 
\midrule
Disease  & 1,044    & 1,043    & 2       & 1,000    & 0         \\
Airport  & 3,188    & 18,631   & 4       & 4        & 1         \\
PubMed   & 19,717   & 44,338   & 3       & 500      & 3.5       \\
CiteSeer & 3,327    & 4,732    & 6       & 3,703    & 5         \\
Cora     & 2,708    & 5,429    & 7       & 1,433    & 11        \\
\bottomrule
\end{tabular}}
\label{tb:summary-dataset}
\end{table}

\reone{
For heterophilic datasets, we evaluate node classification on three benchmarks, respectively, \textsc{Cornell}, \textsc{Texas} and \textsc{Wisconsin} \cite{pei2020geom} from the WebKB dataset (webpage networks). Detailed statistics are summarized in Tab.~\ref{tb:summary-dataset-hetro}. We use the original fixed 10 split datasets. In Tab.~\ref{tb:summary-dataset-hetro}, we report the homophily level $\mathcal{H}$ of each dataset, a sufficiently low $\mathcal{H}\le 0.3$ means that the dataset is more heterophilic when most of neighbours are not in the same class.
}

\begin{table}[htb]
\centering
\caption{Statistics of heterophilic benchmark datasets.}
\resizebox{0.9\linewidth}{!}{\begin{tabular}{lrrrrr} 
\toprule
Dataset  & \# Nodes & \# Edges & Classes & Features & $\mathcal{H}$  \\ 
\midrule
\text{Texas}  & 183    & 295    & 5       & 1,703    & 0.11         \\
\text{Cornell}  & 182    & 295   & 5       & 1,703        & 0.21         \\
\text{Wisconsin}   & 251   & 499   & 5       & 1,703      & 0.30       \\
\bottomrule
\end{tabular}}
\label{tb:summary-dataset-hetro}
\end{table}

\subsubsection{Baselines} Both (1) \textit{Euclidean-hyperbolic comparison} and (2) \textit{deep model comparison} are conducted in the experiment. For (1), we compare our model to 2 feature-based models, 4 Euclidean and 4 hyperbolic graph-based models. Feature-based models: without utilizing graph structure, we feed node feature into MLP and its hyperbolic variant HNN \cite{ganea2018hyperbolic} to predict node labels or links. Graph-based models: we consider GCN \cite{kipf2016semi}, GAT \cite{velivckovic2017graph}, GraphSAGE \cite{hamilton2017inductive}, and SGC \cite{wu2019simplifying} as Euclidean GNN methods. We consider HGNN \cite{liu2019hyperbolic}, HGCN \cite{chami2019hyperbolic}, HGAT \cite{gulcehre2018hyperbolic} and HyboNet \cite{chen2021fully} as hyperbolic variants GNNs. For (2), we compare our model to the state-of-the-art deep GCN model GCNII \cite{chen2020simple}, and also show the performance of vanilla GCN and HGCN under different layer settings.

\subsubsection{Parameter Settings} Under homophilic setting, for node classification, the train/val/test percentage for Airport is 70/15/15\% and standard splits \cite{kipf2016semi, yang2016revisiting} on citation networks. For link prediction, we employ 85/5/10\% edge splits on all datasets. \reone{For heterophilic datasets, the splits are in 60/20/20\% and use all 10 random splits to compute the averaged results.} All models are 16-dimension to ensure a fair comparison. We use Adam for training DeepHGCN, for other methods, we use the Adam and Riemannian Adam optimizer, respectively, for Euclidean and hyperbolic models. We set $\alpha$ to 0.1 and $\beta_l$ for the $l$-th layer to $\log(\frac{\lambda}{l} + 1)$ following \cite{chen2020simple}. The other hyperparameters are obtained via grid search, where $\lambda$: [0.4, 0.5, 0.6], $\gamma$: [1e-3, 1e-4, 1e-5], weight decay: [1e-3, 5e-4, 1e-4], and dropout: [0.0-0.6]. We conduct experiments averaging 10 sets of embeddings' quality using different random seeds. More details are listed in the Appendix.

\begin{table}[htb]
    \centering
    \caption{Test accuracy of different 16-dim models with various depth. Numbers in \textit{bold} denote the best of each model, and in \textit{red bold} highlight the best models for each dataset.}
        \resizebox{\linewidth}{!}{\begin{tabular}{@{}clccccc@{}}
\toprule
                                                               &                                  & \multicolumn{5}{c}{\textbf{Layers}}                                                                                                                                                                                                   \\ \cmidrule(l){3-7} 
\multirow{-2}{*}{\textbf{}}                                    & \multirow{-2}{*}{\textbf{Model}} & 2                         & 8                                                & 16                                               & 32                                               & 64                                               \\ \midrule
                                                               & GCN                              & 78.1$_{\pm 1.5}$          & \textbf{83.2$_{\pm 4.7}$}                        & 47.0$_{\pm 1.4}$                                 & 47.0$_{\pm 1.4}$                                 & 47.0$_{\pm 1.4}$                                 \\
                                                               & GCNII                            & 86.9$_{\pm 2.1}$          & 88.3$_{\pm 2.1}$                                 & \textbf{88.3$_{\pm 1.2}$}                        & 86.6$_{\pm 1.2}$                                 & 86.6$_{\pm 2.3}$                                 \\
                                                               & HGCN                             & \textbf{89.3$_{\pm 1.2}$} & 86.9$_{\pm 4.3}$                                 & 52.8$_{\pm 8.6}$                                 & 47.2$_{\pm 1.2}$                                 & 44.5$_{\pm 1.6}$                                 \\
\multirow{-4}{*}{\rotatebox[origin=c]{90}{\textbf{Airport}}}   & DeepHGCN                         & 89.7$_{\pm 1.9}$          & {\color[HTML]{FE0000} \textbf{94.7$_{\pm 0.9}$}} & 93.9$_{\pm 1.2}$                                 & 93.1$_{\pm 0.9}$                                 & 92.9$_{\pm 1.3}$                                 \\ \midrule
                                                               & GCN                              & \textbf{78.1$_{\pm 0.5}$} & 40.7$_{\pm 0.0}$                                 & 40.7$_{\pm 0.0}$                                 & 40.7$_{\pm 0.0}$                                 & 40.7$_{\pm 0.0}$                                 \\
                                                               & GCNII                            & 77.9$_{\pm 0.8}$          & 76.9$_{\pm 2.5}$                                 & 78.5$_{\pm 1.0}$                                 & {\color[HTML]{FE0000} \textbf{79.4$_{\pm 0.5}$}} & 79.1$_{\pm 0.8}$                                 \\
                                                               & HGCN                             & \textbf{76.5$_{\pm 0.6}$} & 56.0$_{\pm 5.2}$                                 & 54.2$_{\pm 7.5}$                                 & 45.3$_{\pm 6.9}$                                 & 43.7$_{\pm 1.5}$                                 \\
\multirow{-4}{*}{\rotatebox[origin=c]{90}{\textbf{PubMed}}}    & DeepHGCN                         & 78.0$_{\pm 0.4}$          & 78.1$_{\pm 0.7}$                                 & 78.5$_{\pm 0.5}$                                 & \textbf{79.4$_{\pm 0.9}$}                        & 79.0$_{\pm 0.6}$                                 \\ \midrule
                                                               & GCN                              & \textbf{71.0$_{\pm 1.0}$} & 18.1$_{\pm 0.0}$                                 & 18.1$_{\pm 0.0}$                                 & 18.1$_{\pm 0.0}$                                 & 18.1$_{\pm 0.0}$                                 \\
                                                               & GCNII                            & 54.7$_{\pm 7.4}$          & 56.1$_{\pm 6.1}$                                 & 66.5$_{\pm 6.2}$                                 & \textbf{69.4$_{\pm 1.8}$}                        & 68.6$_{\pm 1.8}$                                 \\
                                                               & HGCN                             & \textbf{68.0$_{\pm 0.6}$} & 39.5$_{\pm 7.1}$                                 & 30.3$_{\pm 3.7}$                                 & 31.2$_{\pm 4.9}$                                 & 24.7$_{\pm 1.3}$                                 \\
\multirow{-4}{*}{\rotatebox[origin=c]{90}{\textbf{CiteSeer}}}  & DeepHGCN                         & 72.7$_{\pm 0.5}$          & 71.8$_{\pm 1.5}$                                 & 71.9$_{\pm 0.8}$                                 & 72.7$_{\pm 0.5}$                                 & {\color[HTML]{FE0000} \textbf{73.3$_{\pm 0.4}$}} \\ \midrule
                                                               & GCN                              & \textbf{81.9$_{\pm 1.1}$} & 31.9$_{\pm 0.0}$                                 & 31.9$_{\pm 0.0}$                                 & 31.9$_{\pm 0.0}$                                 & 31.9$_{\pm 0.0}$                                 \\
                                                               & GCNII                            & 78.0$_{\pm 3.0}$          & 76.4$_{\pm 2.2}$                                 & 77.2$_{\pm 3.5}$                                 & 83.7$_{\pm 0.9}$                                 & {\color[HTML]{FE0000} \textbf{84.0$_{\pm 1.3}$}} \\
                                                               & HGCN                             & \textbf{78.1$_{\pm 0.9}$} & 33.5$_{\pm 5.5}$                                 & 30.8$_{\pm 3.9}$                                 & 24.1$_{\pm 7.3}$                                 & 21.7$_{\pm 7.8}$                                 \\
\multirow{-4}{*}{\rotatebox[origin=c]{90}{\textbf{Cora}}}      & DeepHGCN                         & 82.5$_{\pm 0.5}$          & \textbf{83.6$_{\pm 0.4}$}                        & 82.7$_{\pm 1.0}$                                 & 82.4$_{\pm 0.6}$                                 & 83.1$_{\pm 0.7}$                                 \\ \midrule
                                                               & GCN                              & 40.5$_{\pm 8.0}$          & \textbf{41.6$_{\pm 7.2}$}                        & 41.1$_{\pm 8.1}$                                 & 41.6$_{\pm 7.1}$                                 & 41.1$_{\pm 7.9}$                                 \\
                                                               & GCNII                            & \textbf{72.9$_{\pm 4.1}$} & 71.8$_{\pm 4.5}$                                 & 71.8$_{\pm 5.0}$                                 & 69.9$_{\pm 5.8}$                                 & 72.4$_{\pm 4.3}$                                 \\
                                                               & HGCN                             & \textbf{62.1$_{\pm 3.7}$} & 60.3$_{\pm 4.1}$                                 & 58.7$_{\pm 4.4}$                                 & 57.2$_{\pm 4.6}$                                 & 55.9$_{\pm 4.9}$                                 \\
\multirow{-4}{*}{\rotatebox[origin=c]{90}{\textbf{Cornell}}}   & DeepHGCN                         & 75.9$_{\pm 3.6}$          & 74.3$_{\pm 3.8}$                                 & {\color[HTML]{FE0000} \textbf{79.7$_{\pm 3.8}$}} & 76.0$_{\pm 5.1}$                                 & 77.4$_{\pm 4.3}$                                 \\ \midrule
                                                               & GCN                              & 41.1$_{\pm 7.5}$          & 57.8$_{\pm 5.6}$                                 & \textbf{57.8$_{\pm 5.8}$}                        & 55.1$_{\pm 5.6}$                                 & 57.8$_{\pm 5.6}$                                 \\
                                                               & GCNII                            & 80.9$_{\pm 7.1}$          & 78.1$_{\pm 4.3}$                                 & 79.3$_{\pm 5.2}$                                 & \textbf{81.6$_{\pm 6.3}$}                        & 75.7$_{\pm 8.7}$                                 \\
                                                               & HGCN                             & 48.1$_{\pm 6.1}$          & \textbf{53.3$_{\pm 3.7}$}                        & 50.5$_{\pm 5.2}$                                 & 50.4$_{\pm 5.8}$                                 & 45.6$_{\pm 7.3}$                                 \\
\multirow{-4}{*}{\rotatebox[origin=c]{90}{\textbf{Texas}}}     & DeepHGCN                         & 76.3$_{\pm 5.4}$          & 78.9$_{\pm 3.8}$                                 & 81.9$_{\pm 6.2}$                                 & {\color[HTML]{FE0000} \textbf{82.2$_{\pm 5.3}$}} & 80.8$_{\pm 3.6}$                                 \\ \midrule
                                                               & GCN                              & 43.1$_{\pm 3.7}$          & 47.8$_{\pm 8.2}$                                 & \textbf{49.0$_{\pm 8.1}$}                        & 47.9$_{\pm 9.4}$                                 & 47.8$_{\pm 8.8}$                                 \\
                                                               & GCNII                            & 83.7$_{\pm 7.2}$          & {\color[HTML]{FE0000} \textbf{84.1$_{\pm 3.9}$}} & 83.1$_{\pm 1.9}$                                 & 83.5$_{\pm 3.5}$                                 & 80.4$_{\pm 4.3}$                                 \\
                                                               & HGCN                             & \textbf{57.9$_{\pm 5.8}$} & 55.7$_{\pm 6.3}$                                 & 54.2$_{\pm 6.7}$                                 & 51.1$_{\pm 7.4}$                                 & 48.5$_{\pm 8.0}$                                 \\
\multirow{-4}{*}{\rotatebox[origin=c]{90}{\textbf{Wisconsin}}} & DeepHGCN                         & 79.2$_{\pm 4.2}$          & 81.3$_{\pm 5.5}$                                 & \textbf{84.0$_{\pm 4.4}$}                        & 83.7$_{\pm 5.7}$                                 & 82.1$_{\pm 5.3}$                                 \\ \bottomrule
\end{tabular}}
\label{tb:summary-deep}
\end{table}
\subsection{Experiment Results}

\subsubsection{Comparison with Hyperbolic Models}
In Tab.~\ref{tb:summary-lpnc}, we report the averaged ROC AUC for link prediction and F1 score for node classification on various datasets. The dimensions for all models are (16) except for HyboNet (64) since it cannot be trained stably with low embedding dimension. We present our model with the hyperparameter settings that produce the best outcomes. Compared with baselines, our proposed DeepHGCN achieves the best NC and LP performance among all datasets with high hyperbolicity $\delta$. On the Cora dataset with low $\delta$, the DeepHGCN does not outperform the attention-based model on LP and only outperforms the Euclidean GCN on NC by a small margin, trading off the expense of training time. This suggests attention and Euclidean geometry are more suitable for non-hierarchical graph structures.

\subsubsection{Comparison with Deep Models}

In Tab.~\ref{tb:summary-deep}, we evaluate the deep models with different numbers of layers. Instead of using 64-dimensional hidden layers like \cite{chen2020simple}, we apply 16-dimensions for all models. We observe that the performance of GCN and HGCN rapidly degrades when the number of layers surpasses 8, indicating that they are susceptible to over-smoothing. On the other hand, the GCNII does not outperform the shallow models with 2 layers, while consistently improving with more layers and achieving the best result on Cora and PubMed. \reone{The DeepHGCN can not only perform well with 2 layers, but outperforms all models on Airport and CiteSeer with $8 \sim 64$ layers with competitive results. Under the heterophilic datasets Cornell, Texas, and Wisconsin, the DeepHGCN also demonstrates strong performance with deeper architectures, which achieves the highest accuracy of 79.7\% on Cornell with 16 layers, 82.2\% on Texas with 32 layers, and 84.0\% on Wisconsin with 16 layers.} These results suggest our model can alleviate the over-smoothing issue of vanilla HGCN and capable of retrieving information from higher-order neighbors.

\begin{figure}[htb]
    \centering
\begin{center}
\subfigure{
\centering
\label{Fig.contribution.sub.1}
\includegraphics[width=.315\linewidth]{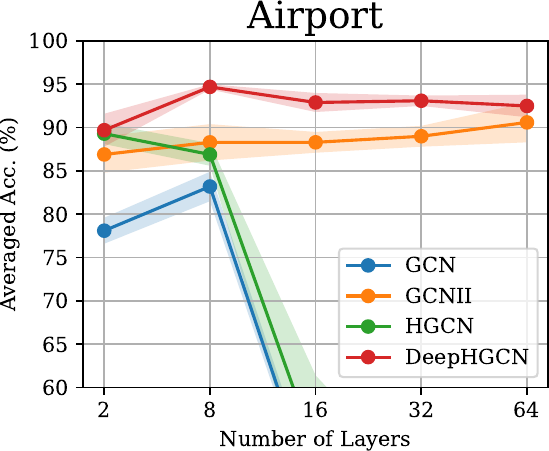}
}
\hspace{-0.4cm}
\subfigure{
\centering
\label{Fig.contribution.sub.2}
\includegraphics[width=.315\linewidth]{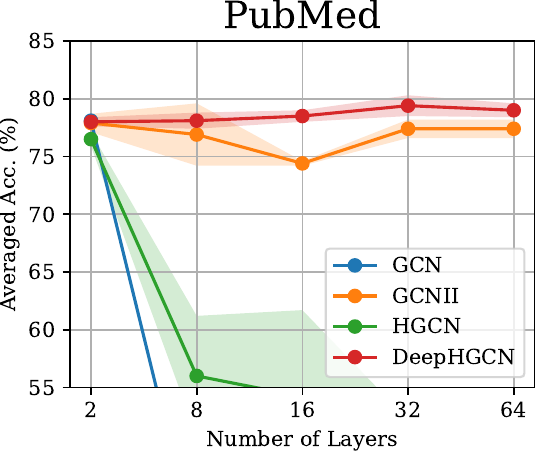}
}
\hspace{-0.4cm}
\subfigure{
\centering
\label{Fig.contribution.sub.3}
\includegraphics[width=.315\linewidth]{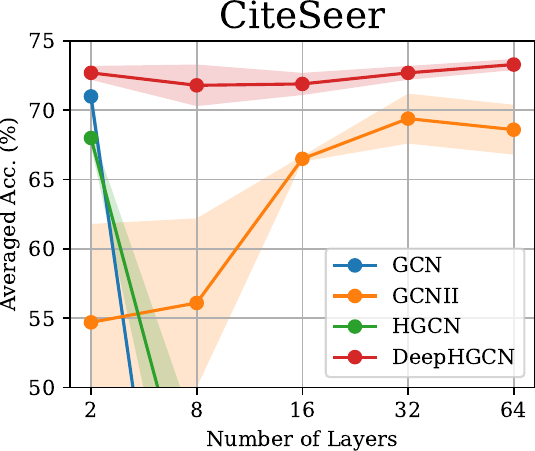}
}

\subfigure{
\centering
\label{Fig.contribution.sub.4}
\includegraphics[width=.315\linewidth]{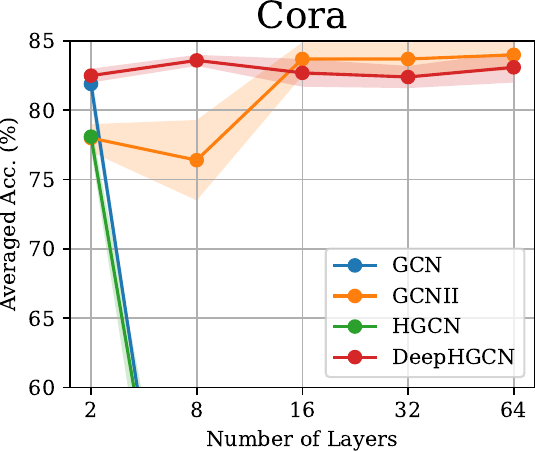}
}
\hspace{-0.4cm}
\subfigure{
\centering
\label{Fig.contribution.sub.5}
\includegraphics[width=.315\linewidth]{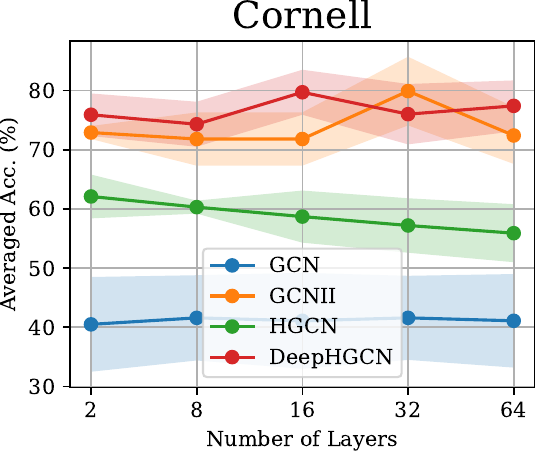}
}
\hspace{-0.4cm}
\subfigure{
\centering
\label{Fig.contribution.sub.6}
\includegraphics[width=.315\linewidth]{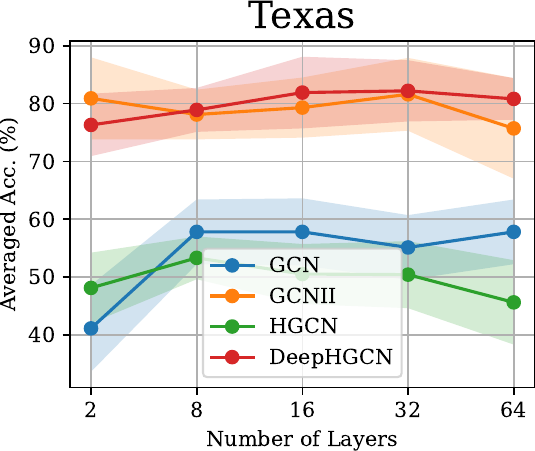}
}
\end{center}
\caption{\reone{Averaged performance of different models with various numbers of layers. We include hyperbolic, homophilic, and heterophilic datasets. The deeper models that overcome over-smoothing generally perform better.}}
\label{Fig.layer_contribution}
\end{figure}
\reone{
\textit{Remark on benefit and limitation of depth.} In Fig.~\ref{Fig.layer_contribution} illustrates the performance comparison through layers. It is worth noting that models with $2$ layers are not sufficient for achieving the best performance. Our DeepHGCN achieves strong performance among groups with generally $8\sim 32$ layers, similar to GCNII according to our validation. As $8$-layer is already considered deep for graph models, we conclude that depth is beneficial for HGCN performance. By carefully tuning the depth, DeepHGCN could achieve significantly better performance than its $2$-layer baselines. However, on some datasets like Texas, the performance of all models declines with added depth beyond a certain point (\textit{e.g} $32\to 64$ layers). This suggests that while our approach helps, extremely deep models may have diminishing returns or even detriments on some datasets. This suggests that DeepHGCN is a remedy rather than a solution to hyperbolic model over-smoothing. Further research is needed to fully address the limitations of over-smoothness and realize the full potential in hyperbolic graph learning.
}

\subsubsection{Over-Smoothing Analysis}
\begin{figure}[h]
    \centering
\begin{center}
\subfigure{
\centering
\label{Fig.energy.sub.1}
\includegraphics[width=.475\linewidth]{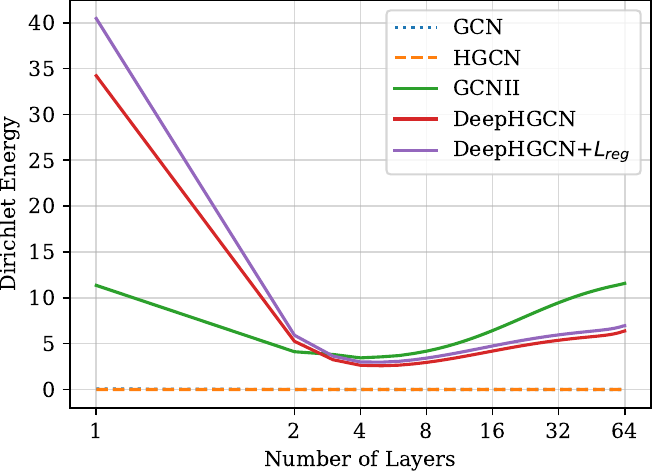}
}
\hspace{-0.2cm}
\subfigure{
\centering
\label{Fig.energy.sub.2}
\includegraphics[width=.465\linewidth]{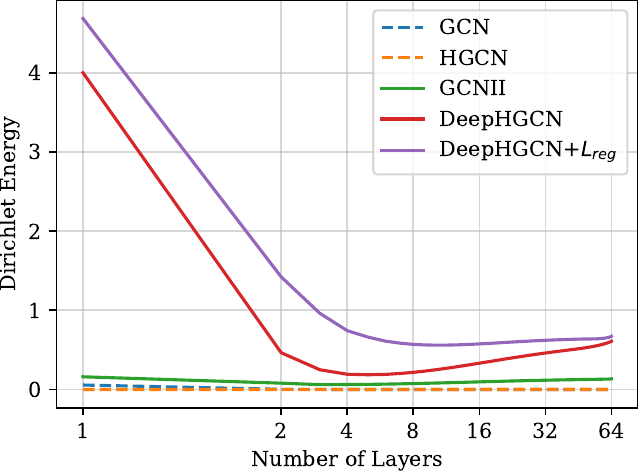}
}
\end{center}
\caption{Dirichlet energy variation through layers of different models on Cora (\textbf{Left}) and CiteSeer (\textbf{Right}). \reone{The x-axis is plotted in log-scale with base 2.}}
\label{Fig.dirichlet-variation}
\end{figure}
In Fig.~\ref{Fig.dirichlet-variation}, we show the Dirichlet energy at each layer of a 64-layer DeepHGCN, comparing with GCN, HGCN and GCNII. Due to the over-smoothing issue, the energy of node embeddings in GCN and HGCN converges rapidly to zero. By fine-tuning the initial residual and weight alignment coefficients, the proposed DeepHGCN is able to obtain energy levels comparable to GCNII on Cora and CiteSeer. We also notice the feature regularized model tends to have higher energy through all layers compared with the unregularized ones. This validates our arguments in Sec.~\ref{sec:rethinking}.

\subsubsection{Ablation Study}
\label{sec:ablation_study}
\begin{figure}[ht]
\centering
\subfigure{
\centering
\label{Fig.ablation.sub.1}
\includegraphics[page=1,width=.47\linewidth]{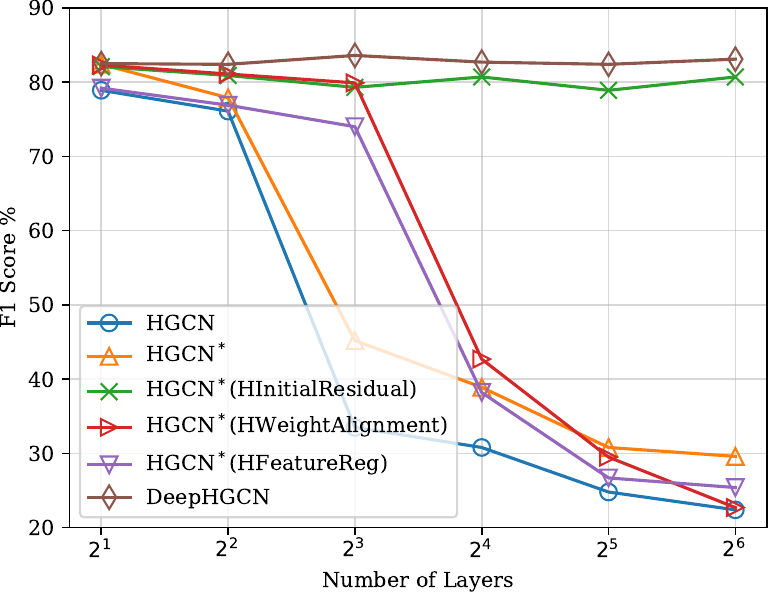}
}
\hspace{-0.2cm}
\subfigure{
\centering
\label{Fig.ablation.sub.2}
\includegraphics[page=1,width=.47\linewidth]{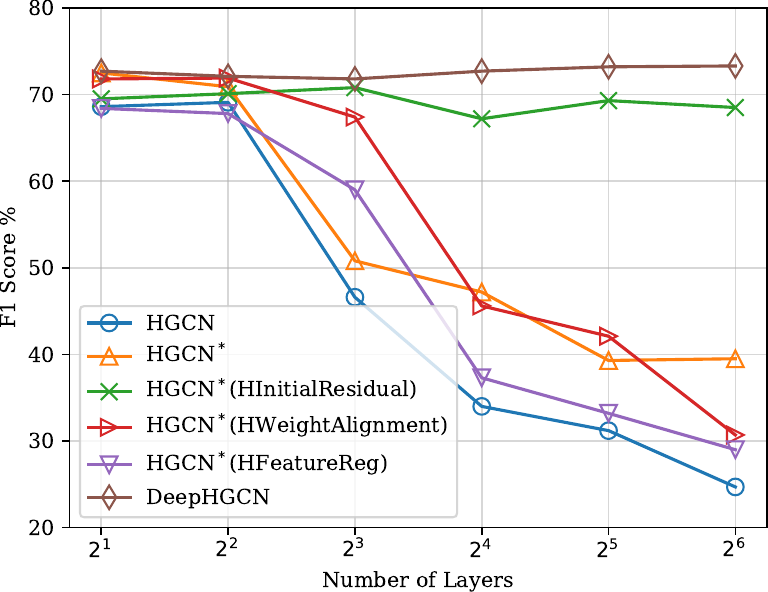}
}
\caption{Ablation study on hyperbolic initial residual, weight alignment and feature regularization on Cora (\textbf{Left}) and CiteSeer (\textbf{Right}) dataset. HGCN$^*$ denotes the HGCN with the proposed efficient backbone.}
\label{Fig.ablation-variation}
\end{figure}
\begin{figure}[h]
\centering
\begin{center}
\subfigure{
\centering
\label{Fig.moduleablation.sub.1}
\includegraphics[width=.54\linewidth]{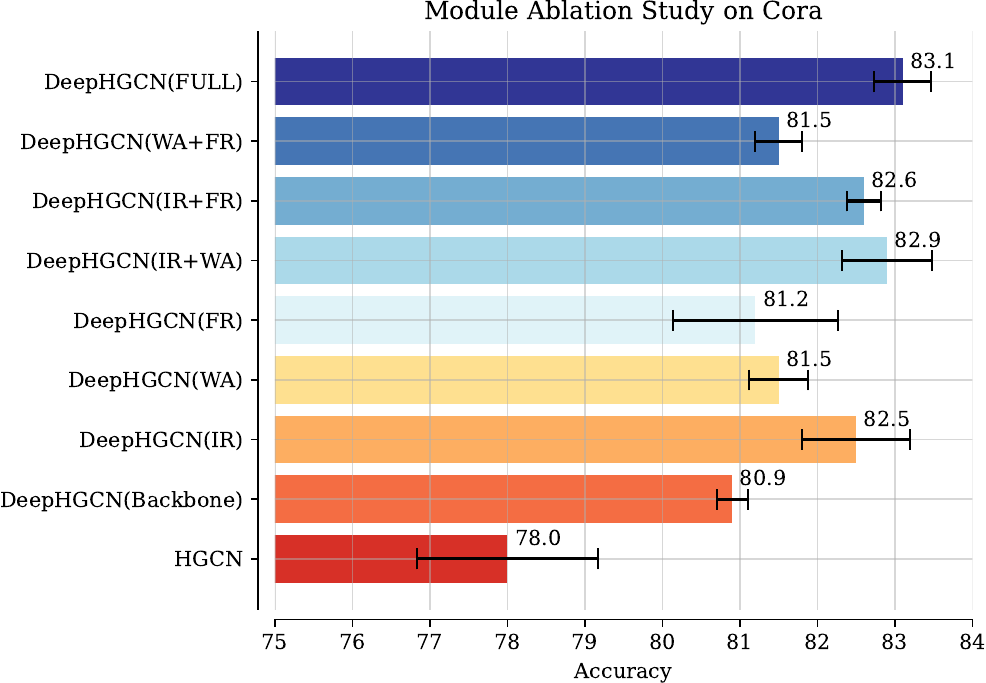}
}
\hspace{-0.2cm}
\subfigure{
\centering
\label{Fig.moduleablation.sub.2}
\includegraphics[width=.4\linewidth]{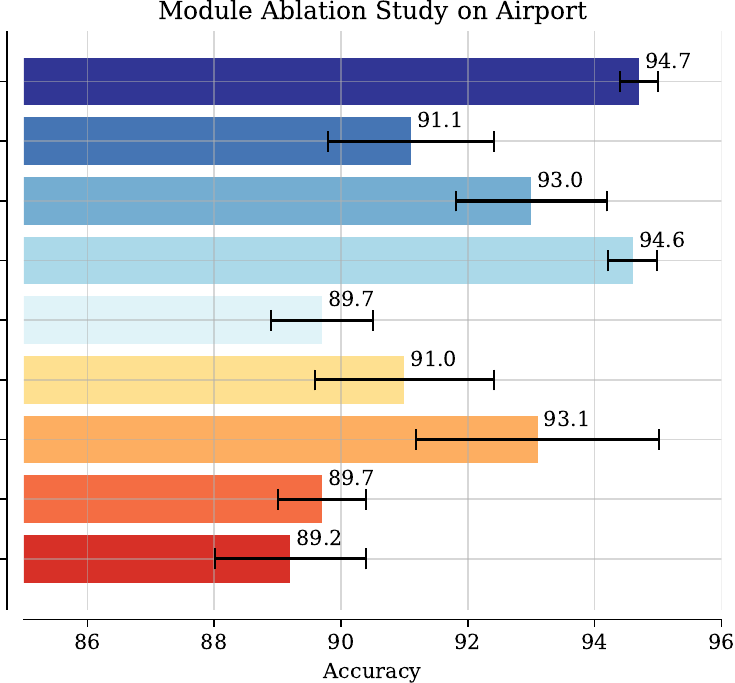}
}
\end{center}
\caption{Performance of permuted components on Cora (\textbf{Left}) and Airport (\textbf{Right}) in $10$ evaluations, we show mean $\pm$ std for each model.}
\label{Fig.component-ablation}
\end{figure}
We study the contributions of our introduced techniques. In Fig.~\ref{Fig.ablation-variation}, we show the performance of DeepHGCN with different depths compared to HGCNs  with one of the proposed techniques applied, further, we provide ablation studies on various component permutation on Airport and Cora datasets in Fig.~\ref{Fig.component-ablation}. We find that hyperbolic initial residual is the most effective module for alleviating over-smoothing, although performance of the 2-layer model still dominates. Directly applying weight alignment and feature regularization also relieves over-smoothing but the models still degrade after depth beyond $2^3$.  This is because a small $\beta_l$ value in Eq.~(\ref{eq:hyperbolic-identity-map}) forces the norm of the weight matrix $\|\mathbf{W}\|$ to be small. According to the theory in \cite{oono2019graph}, the loss of information in an $L$-layer GCN is related to the maximum singular value s of $\mathbf{W}^{(L)}$. A smaller norm $\|\mathbf{W}\|$ implies a smaller $s^L$, thus relieving the loss of information and over-smoothing. However, when the number of layers grows very deep (\textit{e.g.} beyond $2^3$), the multiplicative effect of $\beta$s over many layers ($\prod \beta_l$) may become too small. This could make the model overly reliant on the initial features and not have enough learning capacity, leading to degraded performance. As for feature regularization, applying the technique over a very deep model with many layers may force the node embeddings to be pushed too far apart. This could potentially distort the embed space and make it difficult to preserve the similarities between nodes that are needed for tasks like classification. As a result, performance degrades when the regularization is applied to an overly deep model. Therefore, the DeepHGCN simultaneously apply all three techniques, which assures that performance improves as network depth increases. This suggests that all of the techniques are indispensable for resolving the over-smoothing problem.

\begin{figure}[h]
\centering
\begin{center}
\subfigure{
\centering
\label{Fig.alphacmp.sub.1}
\includegraphics[width=.47\linewidth]{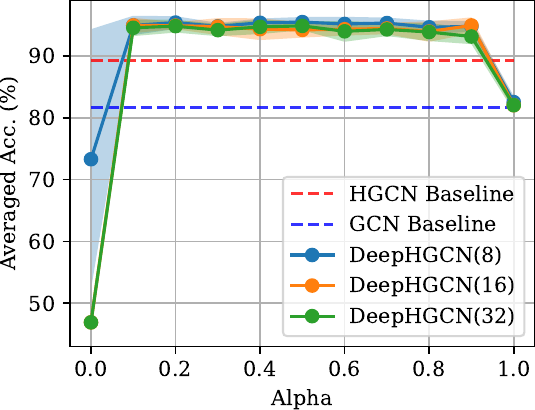}
}
\hspace{-0.2cm}
\subfigure{
\centering
\label{Fig.alphacmp.sub.2}
\includegraphics[width=.47\linewidth]{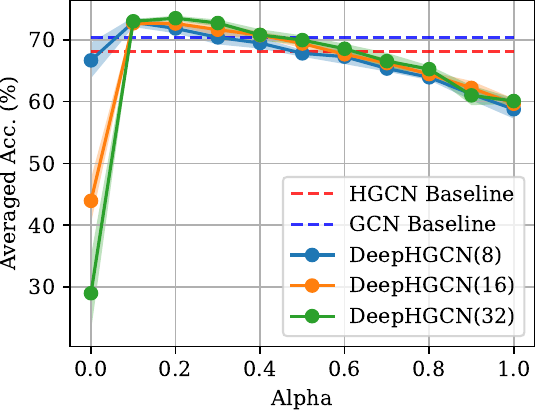}
}
\end{center}
\caption{Ablation study on the impact of different $\alpha_l$ in Eq.~(\ref{eq:hyperbolic-initial-residual}) on Airport (\textbf{Left}) and CiteSeer (\textbf{Right}). A higher $\alpha_l$ value indicates a stronger emphasis on initial embedding, vice versa.}
\label{Fig.alpha-ablation}
\end{figure}
\reone{
Additionally, as the hyperbolic initial residual in Eq.~(\ref{eq:hyperbolic-initial-residual}) contribute significantly to the overall performance, we study the impact of different choices of $\alpha_l$s. With various depth $(8, 16, 32)$ of DeepHGCN backbone, we select $\alpha_l$ from $0$ to $1$ with a step size $0.1$. From Fig.~\ref{Fig.alpha-ablation} we can observe that with Airport, the performance is stably well across $\alpha_l=0.1\sim 0.9$. Under CiteSeer, the performance of DeepHGCN drastically increases when $\alpha_l$ goes from $0$ to $0.1$, then gradually decreases as $\alpha_l$ increases. This suggests that DeepHGCN is adaptive to various $\alpha_l$s, while generally a small portion of initial residual is enough for achieving good performance. 
}

\subsubsection{Embedding Visualization}
\begin{figure}[t]
    \centering
    \includegraphics[width=0.98\linewidth]{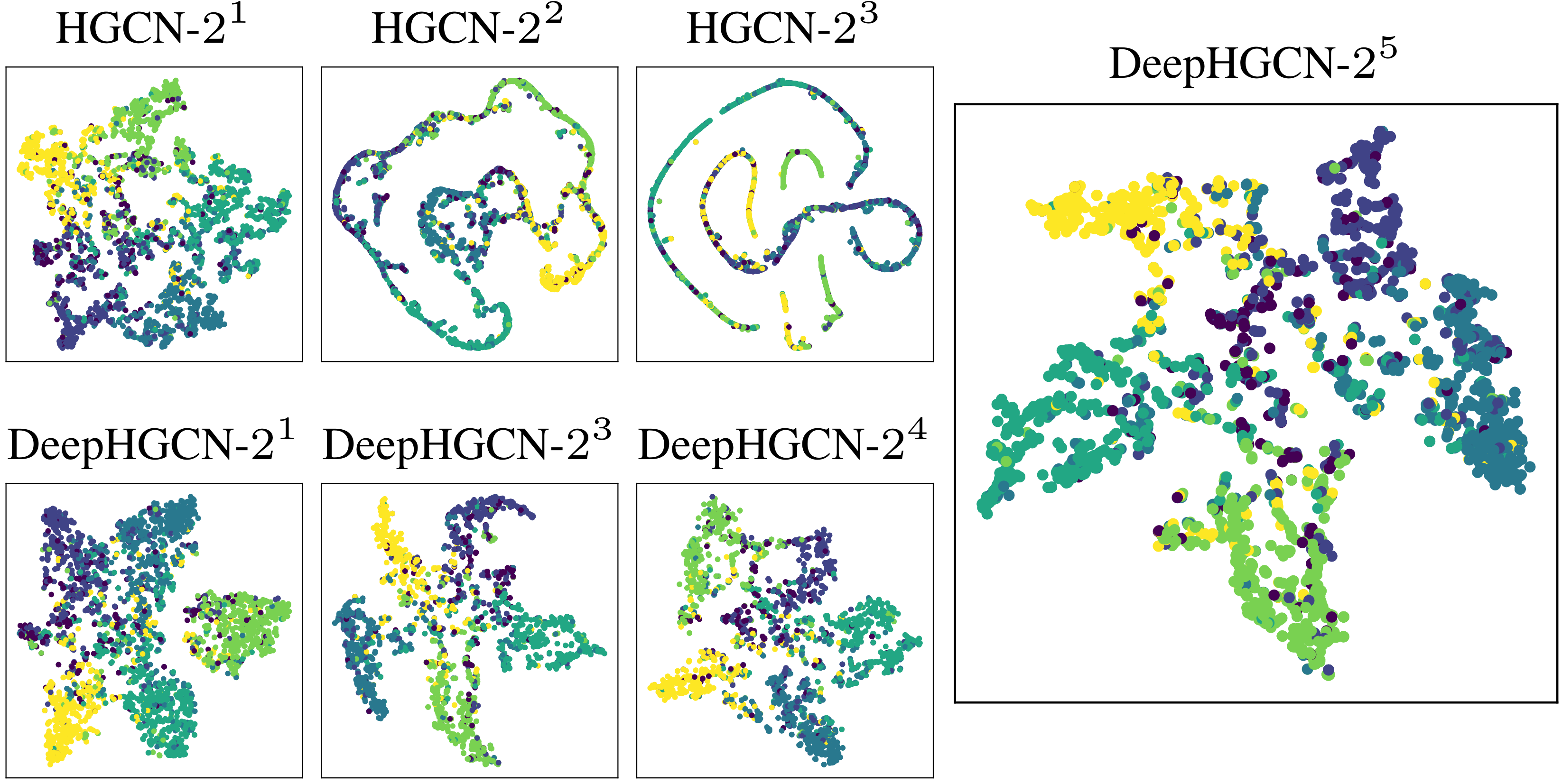}
    \caption{Visualize of node embeddings (after t-SNE) variation through layers on CiteSeer. \reone{Different colors represent different ground truth class labels.}}
\label{Fig.embedding-variation-tsne}
\end{figure}
To better illustrate the effectiveness of DeepHGCN in maintaining distinguishable node representations in deeper layers, we visualize the 16-dimensional node embeddings of both HGCNs and DeepHGCNs using t-SNE \cite{van2008visualizing}. As shown in Fig.~\ref{Fig.embedding-variation-tsne}, as the depth increases, the embeddings of distinct classes in HGCNs tend to converge, resulting in indistinguishable representations. This phenomenon is a clear indication of the over-smoothing problem, which hinders the model's ability to capture meaningful node information in deeper layers. In contrast, DeepHGCN is able to obtain node embeddings that remain well-separated and clearly defined, even in higher layer settings. This evidenced the efficacy of our proposed techniques in alleviating the over-smoothing issue and preserving the discriminative power of node representations in deep hyperbolic networks.

\section{Conclusion}
In this paper, we introduce DeepHGCN, a novel deep hyperbolic graph neural network model that addresses the challenges of developing deeper HGCN architectures while mitigating the over-smoothing problem. Our model is powered by a computationally-efficient and expressive backbone, featuring a fast and accurate hyperbolic linear layer for feature transformation. Additionally, we propose a set of techniques, including hyperbolic initial residual connections, weight alignment, and feature regularization, which work together to effectively prevent over-smoothing while preserving the manifold constraint. Extensive experiments on both hierarchical and non-hierarchical datasets demonstrate that DeepHGCN achieves competitive performance compared to existing Euclidean and shallow hyperbolic GCN models, highlighting the efficacy of our approach in capturing complex graph structures. Future research directions include extending our model to mixed-curvature manifolds and semi-Riemannian manifolds, as well as further investigating techniques to fully resolve the over-smoothing issue in extremely deep architectures.


%

\appendices
\section{Hyperbolic Model Details}
\begin{table*}[h]
\caption{Summary of operations in the Poincar\'e ball model and the Lorentz (hyperboloid) model ($\kappa < 0$)}
\resizebox{\textwidth}{!}{\centering
\begin{tabular}{@{}ccc@{}}
\toprule
                            & \textbf{Poincar\'e Ball Model} $\mathbb{D}$                                                                                                                               & \textbf{Lorentz Model} $\mathbb{L}$                                                                                                                                                                                                  \\
\textbf{Point Set}           & $\mathcal{D}_\kappa^n=\left\{\mathbf{x} \in \mathbb{R}^n:\|\mathbf{x}\|<-\frac{1}{\kappa}\right\}$                                                                    & $\mathcal{L}_\kappa^n=\left\{\mathbf{x} \in \mathbb{R}^{n+1}:\langle \mathbf{x}, \mathbf{x}\rangle_{\mathcal{L}}=\frac{1}{\kappa}\right\}$                                                                                                                      \\ \midrule
\textbf{Metric Tensor}             & $g^{\mathbb{D}_\kappa}_\mathbf{x} = (\lambda_\mathbf{x}^\kappa)^2 g^\mathbb{E}$, where $\lambda_\mathbf{x}^\kappa = \frac{2}{1+\kappa\|\mathbf{x}\|^2}$ and $g^{\mathbb{E}} = \mathbf{I}$                                & $g_\mathbf{x}^{\mathbb{L}_\kappa}=\eta$, where $\eta$ is $I$ except $\eta_{0,0}=-1$                                                                                                                                                       \\ \midrule
\textbf{Geodesic Length}           & $d_\mathbb{D}^\kappa(\mathbf{x},\mathbf{y}) = \frac{2}{\sqrt{|\kappa|}}\tanh^{-1}\left( \sqrt{|\kappa|}\| - \mathbf{x} \oplus_\kappa \mathbf{y}  \| \right)$                     & $d_{\mathbb{L}}^\kappa(\mathbf{x}, \mathbf{y})=\frac{1}{\sqrt{|\kappa|}} \cosh ^{-1}\left(\kappa\langle \mathbf{x}, \mathbf{y}\rangle_{\mathcal{L}}\right)$                                                                                                                 \\ \midrule
\textbf{Exponential Map}    & $\exp_\mathbf{x}^\kappa (\mathbf{v}) = \mathbf{x} \oplus_\kappa \left( \tanh (\sqrt{|\kappa|} \frac{\lambda_\mathbf{x}^\kappa \|\mathbf{v}\|}{2} \frac{\mathbf{v}}{\sqrt{|\kappa|}\|\mathbf{v}\|} ) \right)$ & $\exp _\mathbf{x}^\kappa(\mathbf{v})=\cosh \left(\sqrt{|\kappa|}\|\mathbf{v}\|_{\mathcal{L}}\right) \mathbf{x}+\mathbf{v} \frac{\sinh \left(\sqrt{|\kappa|}\|\mathbf{v}\|_{\mathcal{L}}\right)}{\sqrt{|\kappa|}|| \mathbf{v} \|_{\mathcal{L}}}$                                                          \\ \midrule
\textbf{Logarithmic Map}    & $\log_\mathbf{x}^\kappa (\mathbf{y}) = \frac{2}{\sqrt{|\kappa|}\lambda_\mathbf{x}^\kappa} \tanh^{-1} (\sqrt{|\kappa|}\| -\mathbf{x} \oplus_\kappa \mathbf{y} \|) \frac{ -\mathbf{x} \oplus_\kappa \mathbf{y} }{\| -\mathbf{x} \oplus_\kappa \mathbf{y} \|}$            & $\log _\mathbf{x}^\kappa(\mathbf{y})=\frac{\cosh ^{-1}\left(\kappa\langle \mathbf{x}, \mathbf{y}\rangle_{\mathcal{L}}\right)}{\sinh \left(\cosh ^{-1}\left(\kappa\langle \mathbf{x}, \mathbf{y}\rangle_{\mathcal{L}}\right)\right)}\left(\mathbf{y}-\kappa\langle \mathbf{x}, \mathbf{y}\rangle_{\mathcal{L}} \mathbf{x}\right)$ \\ \midrule
\textbf{Parallel Transport} & $\mathcal{PT}_{\mathbf{x}\to\mathbf{y}}^\kappa (\mathbf{v}) = \frac{\lambda^\kappa_\mathbf{x}}{\lambda^\kappa_\mathbf{y}} \operatorname{gyr}[\mathbf{y}, -\mathbf{x}]\mathbf{v}$                                                                    & $\mathcal{PT}_{\mathbf{x} \to\mathbf{y}}^\kappa(\mathbf{v})=\mathbf{v}-\frac{\kappa\langle \mathbf{y}, \mathbf{v}\rangle_{\mathcal{L}}}{1+\kappa\langle \mathbf{x}, \mathbf{y}\rangle_{\mathcal{L}}}(\mathbf{x}+\mathbf{y})$                                                                                                         \\ \bottomrule
\end{tabular}

}
\label{tb:summary-operation}
\end{table*}

\subsection{Extended Review of Riemannian Geometry}
\label{sec:extend-review-riemannian}
A \text{manifold} $\mathcal{M}$ of $n$-dimension is a topological space that is locally-Euclidean, \textit{i.e.} each point's neighborhood can be approximated by Euclidean space $\mathbb{R}^n$. The tangent space $\mathcal{T}_\mathbf{x} \mathcal{M}$ at $\mathbf{x}\in \mathcal{M}$ is the vector space of all tangent vectors at $\mathbf{x}$, the tangent space is isomorphic to $\mathbb{R}^n$. A Riemannian manifold $(\mathcal{M},g)$ is a manifold $\mathcal{M}$ equipped with Riemannian metric $g = (g_\mathbf{x})_{\mathbf{x}\in \mathcal{M}}$, $g$ is a smooth collection of inner products on the tangent space of $\mathbf{x}\in \mathcal{M}$, \textit{i.e.} $g_\mathbf{x} : \mathcal{T}_\mathbf{x} \mathcal{M} \times \mathcal{M}_\mathbf{x} \mathcal{M} \to \mathbb{R}$. It is natural to deduce Riemannian norm using metric $g$, \textit{i.e.} for any vector $\mathbf{v}\in\mathcal{T}_\mathbf{x}\mathcal{M}$, $\|\mathbf{v}\|_{g_\mathbf{x}} = \sqrt{g_\mathbf{x} (\mathbf{v}, \mathbf{v})}$. The definition of inner-product and the induced norm can induce various geometric notions such as distances between points on $\mathcal{M}$, and angles between vectors on $\mathcal{T}_\mathbf{x}\mathcal{M}$.

\subsubsection{Length of Curve and Geodesic}
In the notion of differential geometry, a \textit{curve} $\gamma$ is defined as a mapping from an interval to the manifold, \textit{i.e.} $\gamma: [a,b]\to \mathcal{M}$. The length of curve is defined as $L(\gamma) = \int_{t_1}^{t_2} \sqrt{g_{\gamma(t)} ( \gamma'(t), \gamma'(t) )} dt$ where $t\in [a,b]$. Given the curve, let $\mathbf{x} = \gamma(a) \in \mathcal{M}$ and $\mathbf{y} = \gamma(b) \in \mathcal{M}$, the minimum distance between $\mathbf{x}$ and $\mathbf{y}$ on the manifold is called the \textit{geodesic}
\begin{equation}
    d(\mathbf{x}, \mathbf{y}) := \inf_\gamma L(\gamma)  = \inf_\gamma \int_{a}^{b} \sqrt{g_{\gamma(t)} ( \gamma'(t), \gamma'(t) )} dt,
    \label{geodesic-general}
\end{equation}
it can be considered as a curve that minimizes the length.

\begin{figure}[h]
\centering
\includegraphics[width=.6\linewidth]{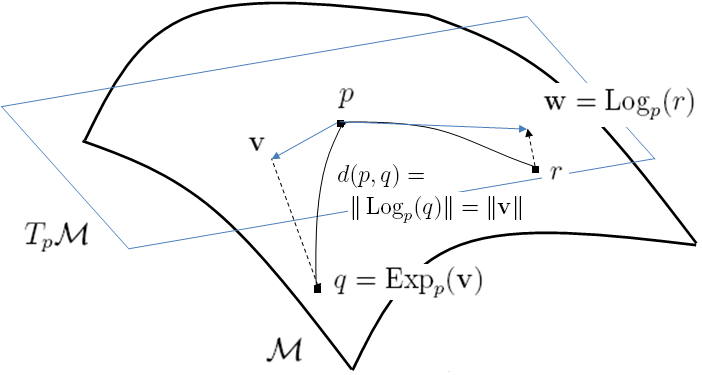}
\caption{Exponential and logarithmic maps of non-Euclidean (Riemannian) manifolds. Illustration from \cite{angulo2014structure}.}
\label{fig:exp-log-map}
\end{figure}

\subsubsection{Exponential and Logarithmic Map}
The exponential map $\exp_\mathbf{x}: \mathcal{T}_\mathbf{x}\mathcal{M} \to \mathcal{M}$ at point $\mathbf{x}$ defines a way to project a vector $\mathbf{v}$ of tangent space $\mathcal{T}_\mathbf{x}\mathcal{M}$ at point $\mathbf{x}$, to a point $\mathbf{y} = \exp_\mathbf{x}(\mathbf{v}) \in \mathcal{M}$ on the manifold. The exponential map is generally used to parameterize a geodesic $\gamma$ uniquely defined by $\gamma(0) = \mathbf{x}$ and $\gamma'(0) = \mathbf{v}$. The logarithmic map $\log_\mathbf{x}: \mathcal{M} \to \mathcal{T}_\mathbf{x}\mathcal{M}$ is the inverse of exponential map, it defines a mapping of an arbitrary vector on $\mathcal{M}$ to the tangent space $\mathcal{T}_\mathbf{x}\mathcal{M}$. Different manifolds have different ways to define exponential maps and logarithmic maps.

\subsubsection{Parallel Transport}
For two points $\mathbf{x},\mathbf{y}\in \mathcal{M}$, the parallel transport $\mathcal{PT}_{\mathbf{x}\to \mathbf{y}}: \mathcal{T}_\mathbf{x}\mathcal{M} \to \mathcal{T}_\mathbf{y}\mathcal{M}$ defines a mapping from a vector $\mathbf{v}$ in $\mathcal{T}_\mathbf{x}\mathcal{M}$ to $\mathcal{T}_\mathbf{y}\mathcal{M}$ that moves $\mathbf{v}$ along the geodesic from $\mathbf{x}$ to $\mathbf{y}$. The parallel transport preserves the metric tensors.

\subsection{Poincar\'e Ball Model}
\label{app:poincare-ball}
The Poincar\'e ball is defined as the Riemannian manifold $\mathbb{D}^n_\kappa = (\mathcal{D}^n_\kappa, g^{\mathbb{D}})$, with point set $\mathcal{D}^n_\kappa = \{\mathbf{x} \in \mathbb{R}^n : \|\mathbf{x}\| < -\frac{1}{\kappa}\}$ and Riemannian metric
\begin{equation}
    g^{\mathbb{D}}_\mathbf{x} = (\lambda_\mathbf{x}^\kappa)^2 g^\mathbb{E},
\end{equation}
where $\lambda_\mathbf{x}^\kappa = \frac{2}{1+\kappa\|\mathbf{x}\|^2}$ (the conformal factor) and $g^{\mathbb{E}} = \mathbf{I}_n$ (the Euclidean metric tensor). $\kappa<0$ is a hyperparameter denoting the sectional curvature of the manifold. 


\subsubsection{Gyrovector Addition}
Existing studies \cite{ganea2018hyperbolic, shimizu2020hyperbolic, mathieu2019continuous} adopt the gyrovector space framework \cite{ungar2005analytic, ungar2008gyrovector} as a non-associative algebraic formalism for hyperbolic geometry. The gyrovector operation $\oplus_\kappa$is termed \textit{M\"obius addition}
\begin{align}
    &\mathbf{x} \oplus_\kappa \mathbf{y} := \frac{(1 - 2\kappa\innerproductcomma{\mathbf{x}}{\mathbf{y}} - \kappa\|\mathbf{y}\|^2)\mathbf{x} + (1 + \kappa\|\mathbf{x}\|^2)\mathbf{y}}{1 - 2\kappa\innerproductcomma{\mathbf{x}}{\mathbf{y}} + \kappa^2 \|\mathbf{x}\|^2 \|\mathbf{y}\|^2}, \label{eq:mobius_add}
\end{align}
where $\mathbf{x},\mathbf{y}\in \mathbb{D}^n_\kappa$. The M\"obius addition defines the addition of two gyrovectors that preserves the summation on the manifold. The induced M\"obius subtraction $\ominus_\kappa$ is defined as $\mathbf{x} \ominus_\kappa \mathbf{y} = \mathbf{x} \oplus_\kappa (-\mathbf{y})$. 

\subsubsection{Tangent Space Operations}
As described in Sec. \ref{sec:bried-review-riemannian}, the \textit{exponential map} projects a vector $\mathbf{v} \in \mathcal{T}_\mathbf{x}\mathbb{D}_\kappa$ in the tangent space at $\mathbf{x}$ to a point on $\mathbb{D}_\kappa$, while the \textit{logarithmic map} projects the manifold vector back to the tangent space
\begin{align}
    &\exp_\mathbf{x}^\kappa (\mathbf{v}) = \mathbf{x} \oplus_\kappa \left( \tanh (\sqrt{|\kappa|} \frac{\lambda_\mathbf{x}^\kappa \|\mathbf{v}\|}{2}) \frac{\mathbf{v}}{\sqrt{|\kappa|}\|\mathbf{v}\|}  \right),\label{eq:expmap}\\
    &\log_\mathbf{x}^\kappa (\mathbf{y}) = \frac{2 \tanh^{-1} (\sqrt{|\kappa|}\| -\mathbf{x} \oplus_\kappa \mathbf{y} \|) (-\mathbf{x} \oplus_\kappa \mathbf{y})}{\sqrt{|\kappa|}\lambda_\mathbf{x}^\kappa \| -\mathbf{x} \oplus_\kappa \mathbf{y} \|}, \label{eq:logmap}
\end{align}
where $\lambda_\mathbf{x}^\kappa = \frac{2}{1+\kappa\|\mathbf{x}\|^2}$ is the \textit{conformal factor} at point $\mathbf{x}$.

\subsubsection{Gyrovector Multiplication}
$\otimes_\kappa$ is the \textit{M\"obius scalar multiplication}. It defines the multiplication between scalar $r$ and a gyrovector. As provided in \cite{ganea2018hyperbolic, chami2019hyperbolic}, the scalar multiplication can be obtained by
\begin{equation}
    r \otimes_\kappa \mathbf{x} = \exp_\mathbf{o}^\kappa (r \log_\mathbf{o}^\kappa (\mathbf{x})),
\label{eq:mscalar-multi-ext}
\end{equation}
where $\mathbf{x}\in \mathbb{D}^n_\kappa/\mathbb{L}^n_\kappa$. One can further extend Eq.~(\ref{eq:mscalar-multi-ext}) to \textit{matrix-vector multiplication}, formulated by
\begin{align}
    &\mathbf{M} \otimes_\kappa \mathbf{x} = \exp_\mathbf{o}^\kappa (\mathbf{M} \log_\mathbf{o}^\kappa (\mathbf{x})), \label{eq:mvector-multi}
\end{align}
where $\mathbf{M}\in \mathbb{R}^{m\times n}$. With broadcasting mechanism \cite{paszke2019pytorch}, we can derive the M\"obius addition and matrix-vector multiplication manipulating on batched representations. 

\subsubsection{Distance Metric}
Since the \textit{geodesic} is the generalized straight line on Riemannian manifold, the distance between two points is essentially the \textit{geodesic length}. For $\mathbf{x}, \mathbf{y} \in \mathbb{D}_\kappa^n$, the distance is given by:
\begin{equation}
    d_\mathbb{D}^\kappa(\mathbf{x},\mathbf{y}) = \frac{2}{\sqrt{|\kappa|}}\tanh^{-1}\left( \sqrt{|\kappa|}\| - \mathbf{x} \oplus_\kappa \mathbf{y}  \| \right).
    \label{eq:poincare-distance}
\end{equation}

\subsection{Lorentz Model}
The Lorentz model, \textit{a.k.a.} the hyperboloid model is defined as the Riemannian manifold $\mathbb{L}^{n}_\kappa = (\mathcal{L}^{n}_\kappa, g^{\mathbb{L}})$, with point set $\mathcal{L}^{n}_\kappa= \{\mathbf{x} \in \mathbb{R}^{n+1}: \innerproductcomma{\mathbf{x}}{\mathbf{x}}_\mathcal{L} = \frac{1}{\kappa}\}$ and Riemannian metric:
\begin{equation}
    g^{\mathbb{L}} = \operatorname{diag}(-1, 1, \cdots, 1).
\end{equation}
The point set $\mathcal{L}^{n}_\kappa$ is geometrically the upper sheet of hyperboloid in an $(n+1)$-dimensional Minkowski space with the origin $(\sqrt{-\frac{1}{\kappa}}, 0,\cdots,0)$. Each point in $\mathbb{L}^{n}_\kappa$ has the form $\mathbf{x} = \begin{bmatrix}
    x_t \\
    \mathbf{x}_s
\end{bmatrix}$, where $x_t\in \mathbb{R}$ is a scalar and $\mathbf{x}_s\in \mathbb{R}^n$. Given $\mathbf{x}, \mathbf{y} \in \mathbb{L}^n_\kappa$, the Lorentzian inner product
\begin{align}
    \innerproductcomma{\mathbf{x}}{\mathbf{y}}_\mathcal{L} &= -x_t y_t + \mathbf{x}^T_s \mathbf{y}_s\\
    & = \mathbf{x}^T \operatorname{diag}(-1, 1, \cdots, 1) \mathbf{y}.
\end{align}

\subsection{Summary of Operations}
\label{app:summary-of-operations}
We summarize the hyperbolic operations of the Poincar\'e ball model and Lorentz model in Tab. \ref{tb:summary-operation}. We denote $\|\cdot\|$ and $\innerproductcomma{\mathbf{x}}{\mathbf{y}}_2$ as the Euclidean L2-norm and inner product, $\innerproductcomma{\mathbf{x}}{\mathbf{y}}_\mathcal{L}$ as the Lorentzian inner product $\mathbf{x}^T \operatorname{diag}(-1, 1, \cdots, 1)\mathbf{y}$ and $\|\cdot\|_\mathcal{L}$ as Lorentzian norm where $\|\mathbf{x}\|_\mathcal{L}^2 = \innerproductcomma{\mathbf{x}}{\mathbf{x}}_\mathcal{L}$. For systematic gyrovector space treatment, please refer to \cite{ungar2008gyrovector}.
\section{Implementation Notes}
\subsection{Fixing Numerical Instability via Clipping}
In our implementation, 32-bit float tensors are used for all manipulations. In practice, computing the square of a float tensor $\mathbf{x}$ requires lower numeric limit, for instance, if $\mathbf{x} = 10^{-a}$, the precision required for $\mathbf{x}^2$ is at least $10^{-2a}$. The smallest value for a 32-bit float tensor is approximately $1.175494 \times 10^{-38}$, thus if $a>19$, $\mathbf{x}^2$ will likely be out of memory and result in Nan value. In such cases, the tensors are out of hyperbolic space and could mislead the training. To avoid numerical instability, we employ feature clipping:
\begin{equation}
\begin{aligned}
    \operatorname{Clip}(\mathbf{x}; a, \kappa) =
\begin{cases}
\frac{1 - \epsilon}{\sqrt{|\kappa|}\|\mathbf{x}\|}\mathbf{x}, &\|\mathbf{x}\| \ge \frac{1 - \epsilon}{\sqrt{|\kappa|}}\\
\frac{a}{\|\mathbf{x}\|}\mathbf{x}, &\|\mathbf{x}\| < a\\
\mathbf{x}, &\text{otherwise}
\end{cases}
\end{aligned}
\end{equation}
where the $a$ is usually fixed to $10^{-15}$. The feature clipping is adopted in many parts of our implementation where we are likely to get Nan values.

\subsection{Dropout}
Dropout is an essential technique for preventing over-fitting in hyperbolic models. In our implementation, we perform dropout on $\boldsymbol{\omega} = \phi(\cdot)$ in Eq.~(\ref{eq:poincare-fc-layer}). In the following, we verify that the hyperbolic representation in the Poincar\'e ball model after dropout is manifold-preserving. 

For arbitrary $\boldsymbol{\omega}$, following Eq.~(\ref{eq:poincare-fc-layer}) we have
\begin{align}
    \|\mathcal{F}_\mathbb{D}^\kappa(\cdot)\| &= \|\frac{\boldsymbol{\omega}}{1+\sqrt{1 - \kappa\|\boldsymbol{\omega}\|^2}}\|\\
    &= \sqrt{\frac{\|\boldsymbol{\omega}\|^2}{(1 + \sqrt{1 - \kappa\|\boldsymbol{\omega}\|^2})^2}}\\
    &= \sqrt{\frac{\|\boldsymbol{\omega}\|^2}{1 + 2\sqrt{2 - \kappa\|\boldsymbol{\omega}\|^2  }- \kappa\|V\|^2}}\\
    &= \sqrt{\frac{\|\boldsymbol{\omega}\|^2}{1 + 2\sqrt{2 + |\kappa|\|\boldsymbol{\omega}\|^2  }+ |\kappa|\|V\|^2}}\\
    &= \sqrt{\frac{1}{\frac{1}{\|\boldsymbol{\omega}\|^2} + 2\sqrt{\frac{2}{\|\boldsymbol{\omega}\|^4} + \frac{|\kappa|}{\|\boldsymbol{\omega}\|^2}  }+ |\kappa|}}\\
    &< \frac{1}{\sqrt{|\kappa|}}\label{eq:dropout-v-final-ineq}.
\end{align}
Notably, $\|\mathcal{F}_\mathbb{D}^\kappa(\cdot)\|$ reaches $\frac{1}{\sqrt{|\kappa|}}$ and $0$ respectively when $\boldsymbol{\omega}$ approximate infinity and when $\boldsymbol{\omega}=0$. Thus the range of each component in $\boldsymbol{\omega}$ is $[-\infty, \infty]$, which is in coincidence with the Euclidean space representation. Hence, the dropout can be directly applied to $\boldsymbol{\omega}$ without further generalization.


\subsection{Non-linear Activation}
The non-linear activation prevents multi-layer GCNs from collapsing into single-layer networks. Activation functions are typically applied after neighborhood aggregation and before linear transformation step for optimal performance \cite{kipf2016semi}. In the Poincar\'e ball model, applying ReLU is manifold-preserving (\textit{i.e.} $\forall \mathbf{x}\in \mathbb{D}_\kappa^n$ we have $\sigma(\mathbf{x})\in \mathbb{D}_\kappa^n$) since ReLU only cut-off the negative half and remain the positive half unchanged.
\section{Proofs and Derivations}

\subsection{Proof of Proposition \ref{lemma:shrinking-property}}
\begin{proof} We start the proof from the definition of hyperbolic Dirichlet energy on the Poincar\'e ball model. Given the closed form solution of the distance function $d_{\mathbb{D}}$, we have
\begin{align}
&f^{\mathbb{D}}_\text{DE}(\tilde{\mathbf{P}}\otimes_\kappa\mathbf{H}) = \label{lemma1-proof-begin}\\
& \frac{1}{2} \sum_{i,j} a_{ij} d_{\mathbb{D}}^2\left( \frac{1}{\sqrt{1+d_i}} \otimes_\kappa\tilde{\mathbf{P}}\otimes_\kappa \mathbf{h}_i^{}  ,\frac{1}{\sqrt{1+d_j}} \otimes_\kappa\tilde{\mathbf{P}}\otimes_\kappa \mathbf{h}_j^{}\right)  \nonumber\\
& = \frac{1}{2} \sum_{i,j} a_{ij} \left(\frac{2}{\sqrt{|\kappa|}}\tanh^{-1}\left( \sqrt{|\kappa|} \|\rho(\mathbf{H})\| \right)\right)^2,\label{property-de-derivatrion-1}
\end{align}
where function $\rho$ can be expanded as
\begin{align}
&\rho(\mathbf{H}) = \nonumber\\
&\left(-\frac{1}{\sqrt{1+d_i}}\otimes_\kappa (\tilde{\mathbf{P}}\otimes_\kappa \mathbf{h}_i)\right)\oplus_\kappa\left(\frac{1}{\sqrt{1+d_j}}\otimes_\kappa (\tilde{\mathbf{P}}\otimes_\kappa \mathbf{h}_j)\right)\\
&=\left(-\frac{\tilde{\mathbf{P}}}{\sqrt{1+d_i}}\otimes_\kappa \mathbf{h}_i\right)\oplus_\kappa\left(\frac{\tilde{\mathbf{P}}}{\sqrt{1+d_j}}\otimes_\kappa \mathbf{h}_j\right)\\
&= \tilde{\mathbf{P}} \left(\left(-\frac{1}{\sqrt{1+d_i}}\otimes_\kappa \mathbf{h}_i\right)\oplus_\kappa\left(\frac{{1}}{\sqrt{1+d_j}}\otimes_\kappa \mathbf{h}_j\right)\right).
\end{align}
One can easily prove that for all $\mathbf{A}\in \mathbb{R}^{n\times n}$ and $\mathbf{x}\in \mathbb{R}^n$, $\|\mathbf{A}\mathbf{x}\|\le \|\mathbf{A}\|_F \|\mathbf{x}\|$, thus the $\|\rho(\tilde{\mathbf{P}}, \mathbf{H}^{(l)})\|$ in Eq. (\ref{property-de-derivatrion-1}) can be further written as
\begin{align}
    &\|\rho(\mathbf{H})\|\nonumber\\
    &= \left\|\tilde{\mathbf{P}} \left(\left(-\frac{1}{\sqrt{1+d_i}}\otimes_\kappa \mathbf{h}_i\right)\oplus_\kappa\left(\frac{1}{\sqrt{1+d_j}}\otimes_\kappa \mathbf{h}_j\right)\right)\right\|\\
    &\le \|\tilde{\mathbf{P}} \|_F \left\|\left(-\frac{1}{\sqrt{1+d_i}}\otimes_\kappa \mathbf{h}_i\right)\oplus_\kappa\left(\frac{1}{\sqrt{1+d_j}}\otimes_\kappa \mathbf{h}_j\right)\right\|.
\end{align}
Since $\tilde{\mathbf{P}}$ is normalized (can be row normalize or diagonal normalize), the Frobenius norm $\|\tilde{\mathbf{P}}\|_F\ge 1$ always establish, thus the norm of $\rho$:
\begin{align}
    &\|\rho(\mathbf{H})\| \nonumber\\
    &\le\left\|\left(-\frac{1}{\sqrt{1+d_i}}\otimes_\kappa \mathbf{h}_i\right)\oplus_\kappa\left(\frac{1}{\sqrt{1+d_j}}\otimes_\kappa \mathbf{h}_j\right)\right\|.
\end{align}
Known that $\tanh^{-1}(\cdot)$ is a monotonically increasing function, Eq. (\ref{property-de-derivatrion-1}) can be further derived as
\begin{align}
    & \frac{1}{2} \sum_{i,j} a_{ij} \left(\frac{2}{\sqrt{|\kappa|}}\tanh^{-1}\left( \sqrt{|\kappa|} \|\rho(\tilde{\mathbf{P}}, \mathbf{H}^{})\| \right)\right)^2\\
    & \le \frac{1}{2} \sum_{i,j} a_{ij} \biggr(\frac{2}{\sqrt{|\kappa|}}\tanh^{-1}\biggr( \sqrt{|\kappa|} \biggr\|\biggr(-\frac{1}{\sqrt{1+d_i}}\otimes_\kappa \mathbf{h}_i\biggr)\nonumber \\&
    \qquad\qquad\qquad\qquad \oplus_\kappa  \biggr(\frac{1}{\sqrt{1+d_j}}\otimes_\kappa \mathbf{h}_j\biggr)\biggr\| \biggr)\biggr)^2\\
    &= \frac{1}{2} \sum_{i,j} a_{ij} d_{\mathbb{D}}^2\left( \frac{1}{\sqrt{1+d_i}} \otimes_\kappa \mathbf{h}_i^{}  ,\frac{1}{\sqrt{1+d_j}} \otimes_\kappa \mathbf{h}_j^{}\right)\\
    &= f_\text{DE}^\mathbb{D}(\mathbf{H}) \label{lemma1-proof-end}.
\end{align}
Eq. (\ref{lemma1-proof-begin}-\ref{lemma1-proof-end}) concludes the proof.
\end{proof}
\section{Appendix D: Additional Evaluations}

\subsection{Additional Assessment on DISEASE Dataset}
\begin{table}[t]
    \centering
    \caption{Model evaluation in DISEASE.}
    \resizebox{.7\linewidth}{!}{\begin{tabular}{llcc} 
\toprule
\textbf{Dataset ($\delta$)}     &  & \multicolumn{2}{c}{\textbf{Disease ($\delta = 0$)}}                      \\ 
\cmidrule{3-4}
\textbf{Task}                   &  & LP                       & NC                                            \\ 
\midrule
GCN                             &  & 58.00$\pm 1.41$          & 69.79$\pm 0.54$                               \\
GAT                             &  & 58.16$\pm 0.92$          & 70.04$\pm 0.49$                               \\
GraphSAGE                       &  & 65.93$\pm 0.29$          & 70.10$\pm 0.49$                               \\
SGC                             &  & 65.34$\pm 0.28$          & 70.94$\pm 0.59$                               \\
GCNII (8)                       &  & /                        & 88.83$\pm 1.32$                               \\
GCNII (16)                      &  & /                        & \multicolumn{1}{l}{\textbf{96.71$\pm 2.78$}}  \\ 
\midrule
HGNN                            &  & 81.54$\pm 1.22$          & 81.27$\pm 3.53$                               \\
HGCN                            &  & 90.80$\pm 0.31$          & 88.16$\pm 0.76$                               \\
HGAT                            &  & 87.63$\pm 1.67$          & 90.30$\pm 0.62$                               \\
HyboNet                         &  & \textbf{96.80$\pm 0.40$} & 96.00$\pm 1.00$                               \\ 
\midrule
\textcolor{blue}{DeepHGCN (2)}  &  & 92.10$\pm 0.44$          & 89.90$\pm 1.33$                               \\
\textcolor{blue}{DeepHGCN (8)}  &  & 95.70$\pm 0.32$          & 92.51$\pm 2.10$                               \\
\textcolor{blue}{DeepHGCN (16)} &  & 95.51$\pm 1.52$          & 93.70$\pm 1.52$                               \\
\bottomrule
\end{tabular}}
    \label{tab:disease_dataset}
\end{table}
We provide the performance comparisons of models in Disease dataset in Tab.~\ref{tab:disease_dataset}. We observed that, when increasing the depth of GCNII, the validation accuracy of node classification will reach almost $99\%$, suggesting Euclidean space is also capable for embedding DISEASE. Although the hyperbolic space is more natural for embedding tree-like data and therefore could yield improved performance on the DISEASE dataset, it looks the potential improvement we could expect over GCNII is marginal. 

According to our experiments, GCNII, HyboNet and DeepHGCN are all capable of fitting the data, that is, the training accuracy can reach $100\%$ while other models are unable to fit the data even without dropout and weight regularization. We infer that the performance gap observed in the test set is attributed to the poor generalization ability of hyperbolic classifiers, thereby suggesting an intriguing direction for future research.

\bibliographystyle{IEEEtran}
\bibliography{bibtex/bib/egbib}

\end{document}